\documentclass[9pt,journal]{IEEEtran}

\usepackage[mathscr]{eucal}
\usepackage[cmex10]{amsmath}
\usepackage{epsfig,epsf,psfrag}
\usepackage{amssymb,amsmath,amsthm,amsfonts,latexsym}
\usepackage{amsmath,graphicx,bm,xcolor,url}
\usepackage[caption=false]{subfig} 
\usepackage{fixltx2e}%
\usepackage{array}%
\usepackage{verbatim}
\usepackage{bm}
\usepackage{algorithmic, cite}
\usepackage{algorithm}
\usepackage{verbatim}
\usepackage{textcomp}
\usepackage{mathrsfs}
\usepackage{multirow}
\usepackage{epstopdf}
\catcode`~=11 \def\UrlSpecials{\do\~{\kern -.15em\lower .7ex\hbox{~}\kern .04em}} \catcode`~=13 

\allowdisplaybreaks[1]

\newcommand{\calF}{\mathcal{F}}

\newcommand{\calP}{\mathcal{P}}

\newcommand{\calS}{\mathcal{S}}

\newcommand{\calX}{\mathcal{X}}

\newcommand{\bbR}{\mathbb{R}}

\DeclareMathAlphabet{\mathbsf}{OT1}{cmss}{bx}{n}
\DeclareMathAlphabet{\mathssf}{OT1}{cmss}{m}{sl}%

\DeclareSymbolFont{bsfletters}{OT1}{cmss}{bx}{n}  
\DeclareSymbolFont{ssfletters}{OT1}{cmss}{m}{n}
\DeclareMathSymbol{\bsfGamma}{0}{bsfletters}{'000}
\DeclareMathSymbol{\ssfGamma}{0}{ssfletters}{'000}
\DeclareMathSymbol{\bsfDelta}{0}{bsfletters}{'001}
\DeclareMathSymbol{\ssfDelta}{0}{ssfletters}{'001}
\DeclareMathSymbol{\bsfTheta}{0}{bsfletters}{'002}
\DeclareMathSymbol{\ssfTheta}{0}{ssfletters}{'002}
\DeclareMathSymbol{\bsfLambda}{0}{bsfletters}{'003}
\DeclareMathSymbol{\ssfLambda}{0}{ssfletters}{'003}
\DeclareMathSymbol{\bsfXi}{0}{bsfletters}{'004}
\DeclareMathSymbol{\ssfXi}{0}{ssfletters}{'004}
\DeclareMathSymbol{\bsfPi}{0}{bsfletters}{'005}
\DeclareMathSymbol{\ssfPi}{0}{ssfletters}{'005}
\DeclareMathSymbol{\bsfSigma}{0}{bsfletters}{'006}
\DeclareMathSymbol{\ssfSigma}{0}{ssfletters}{'006}
\DeclareMathSymbol{\bsfUpsilon}{0}{bsfletters}{'007}
\DeclareMathSymbol{\ssfUpsilon}{0}{ssfletters}{'007}
\DeclareMathSymbol{\bsfPhi}{0}{bsfletters}{'010}
\DeclareMathSymbol{\ssfPhi}{0}{ssfletters}{'010}
\DeclareMathSymbol{\bsfPsi}{0}{bsfletters}{'011}
\DeclareMathSymbol{\ssfPsi}{0}{ssfletters}{'011}
\DeclareMathSymbol{\bsfOmega}{0}{bsfletters}{'012}
\DeclareMathSymbol{\ssfOmega}{0}{ssfletters}{'012}

\newcommand{\iid}{i.i.d.\ }

\newcommand{\dotleq}{\stackrel{.}{\leq}}

\newcommand{\dotgeq}{\stackrel{.}{\geq}}

\DeclareMathOperator*{\argmin}{arg\,min}

\newtheorem{thm}{Theorem}
\newtheorem{lemma}{Lemma}

\newtheorem{corollary}{Corollary}
\newtheorem{definition}{Definition}

\newtheorem{remark}{Remark}

\usepackage{pgfplots}
\usepackage{tikz}

\newcommand{\qednew}{\nobreak \ifvmode \relax \else
      \ifdim\lastskip<1.5em \hskip-\lastskip
      \hskip1.5em plus0em minus0.5em \fi \nobreak
      \vrule height0.75em width0.5em depth0.25em\fi}

\newcommand{\KL}[2]{\mathrm{D}\big(#1  \| #2 \big)}

\newcommand{\GJSname}{\mathrm{GJS}}
\newcommand{\GJS}[3]{\GJSname\left(#1, #2, #3\right)}

\newcommand{\ED}[1]{\widehat{Q}_{#1}}
\newcommand{\ent}[1]{\mathrm{H}\left(#1\right)}
\usepackage{xspace}

\usepackage[ colorlinks = true,
             linkcolor = blue,
             urlcolor  = blue,
             citecolor = red,
            anchorcolor = green,
			]{hyperref}

\usepackage{cite}
\usepackage{mathtools}
\allowdisplaybreaks[2]
\let\iid\undefined
\newcommand{\iid}{i.i.d.\@\xspace}

\usepackage{bbm}

\newcommand{\pe}{\mathrm{P}^{\mathrm{err}}}
\newcommand{\probm}{\mathrm{P}}
\begin{document} 

\title{Sequential Classification with \\ Empirically Observed Statistics}
 
\author{ Mahdi~Haghifam ,~\IEEEmembership{Member,~IEEE,} Vincent~Y.~F.~Tan, ~\IEEEmembership{Senior Member,~IEEE,} and Ashish~Khisti,~\IEEEmembership{ Member,~IEEE,}
\thanks{This paper was presented partially at the 2019 IEEE Information Theory Workshop (ITW)\cite{haghifam2019sequential}. MH and AK are with the Department of Electrical and Computer Engineering, University of Toronto (emails: mahdi.haghifam$@$mail.utoronto.ca, akhisti$@$ece.utoronto.ca). VYFT is with the Department of Electrical and Computer Engineering and the Department of Mathematics, National  University of Singapore (email: vtan$@$nus.edu.sg)}  %
}

\maketitle
\begin{abstract}
Motivated by real-world machine learning applications, we consider a statistical classification task in a sequential setting where test samples arrive sequentially. In addition,   the generating distributions are unknown and only a set of empirically sampled sequences are available to a decision maker.
The decision maker is tasked to classify a test sequence which is known to be generated according to either one of the distributions. In particular, for the binary case, the decision maker wishes to perform the classification task with minimum number of the test samples, so, at each step, she declares that either hypothesis $1$ is true, hypothesis $2$ is true, or she requests for an additional test sample. We propose a classifier and analyze the type-I and type-II error probabilities. We demonstrate the significant advantage of our sequential scheme compared to an existing non-sequential classifier proposed by Gutman. Finally, we extend our setup and results to the multi-class classification scenario and again demonstrate that the variable-length nature of the problem affords significant advantages as one can achieve the same set of exponents as Gutman's fixed-length setting but without having the rejection option.
\end{abstract}
\section{Introduction}
\IEEEPARstart{Q}{uick} and accurate classification is crucial in many real-life applications. For instance, to diagnose haematologic diseases based on blood test results, a physician wishes to detect the pattern, deviations, and relations in the blood samples of a patient as quickly as possible to make treatment plans. Similar challenges can be found in a broad range of applications
such as genomics analysis, finance, and abnormal detection where there is an inherent trade-off between speed and accuracy. \par
In many real-world applications,  classical \textit{hypothesis testing} is infeasible due to the fact that the probability distributions
of the sources are unknown. In practice, one often encounters {\em classification} problems in which one has access to training samples and is required to {\em classify} a set of test samples according to which distribution this set is generated from. To incorporate the real-life requirement of classifying the test samples as quickly as possible, one can consider the \textit{sequential statistical classification} setup. This setup addresses the problem of classifying test samples given training samples with the additional requirement that the decision maker is required to make his/her decision based on as {\em   few} tests samples as possible; it is however, known that all the test samples originate from the same  distribution. %
\par 
The problem of classification using empirically observed statistics has been studied in many prior works. Gutman in \cite{gutman1989asymptotically} formulated a problem in which a decision maker has access to two training sequences which are generated according to two distinct and unknown distributions. Then, a fixed-length test sequence is given to the decision maker, and the decision maker is tasked to classify the test sequence. For this problem, Gutman proposes an asymptotically optimal test. The results in \cite{gutman1989asymptotically} are obtained in the asymptotic regime when the length of the training sequences tends to infinity. In this regard, the non-asymptotic and second-order performance of the Gutman's test is analyzed in \cite{zhousecond} where it is shown that Gutman's test is, in fact, second-order optimal. Moreover, Ziv~\cite{ziv1988classification} studied the relationship between test rules for the binary classification problem and universal data compression methods.  Unnikrishnan and Naini~\cite{deanonym} and Unnikrishnan~\cite{unnikrishnan2015asymptotically} extended Gutman's proposed test for the case with multiple test sequences and obtain an optimal test rule for a certain matching task between multiple test sequences. Furthermore, Unnikrishnan and Huang in \cite{unnikrishnan2016weak} showed  how one can apply the results on the weak convergence of the test statistic to obtain better approximations for the error probabilities for   statistical classification in the finite sample size setting. Kelly et al. \cite{kelly2013classification} considered the classification problem with empirically observed statistics for large alphabet sources. They consider a scenario in which the alphabet size grows
with the length of the training and the test sequences, and the authors characterized the maximum growth rate of the alphabet size for which consistent classification is
possible. The related problem of closeness testing has been investigated in \cite{acharya12, valiant2017automatic}. Another related problem in this area is estimating properties of distributions  using empirically observed statistics. This problem has been considered in various setups such as estimation of the support of distribution \cite{orlitsky2016optimal,zou2016quantifying} and the  estimation of the order of a finite-state Markov chain \cite{merhav1989estimation}, etc. Recently, Acharya et al.\ in \cite{acharya2016unified} proposed an optimal method for the estimation of certain properties of  distribution using empirically  observed statistics which is applicable for a wide range of property estimation problems. The authors in \cite{he2020distributed} studied the problem of distributed detection in the setting that the central node has access to noisy test and training sequences. Finally, \cite{hsu2020binary} considered the Gutman's setup with the difference is that there is a mismatch between the generating distribution of the test sequence and that of the training sequences, which they called ``mismatch". In this setup an optimal classifier was proposed in \cite{hsu2020binary}.   \par
In this paper, we consider an information-theoretic formulation of sequential classification. Recall that in the simple sequential binary hypothesis testing scenario, a decision maker is given a variable-length test sequence and knows that it is either generated in an i.i.d.\ fashion from one of the \textit{known} distributions $P_1$ or $P_2$. It is well-known that the sequential probability ratio test (SPRT) is optimal for  sequential binary hypothesis testing \cite{wald1945sequential}. However, we consider a scenario that the decision maker {\em does not know  both} generating distributions, i.e., $P_1$ and $P_2$. Instead the decision maker has access to two fixed-length training sequences, one is drawn \iid from $P_1$ and the other \iid from $P_2$. Then, the task of the decision maker is to classify a test sequence which is drawn \iid  from either $P_1$ or $P_2$. The decision maker observes the test sequence sequentially and may choose when to stop sampling once she is sufficiently confident. At that time, she makes a final decision. Also, we extend our framework beyond the binary classification setting and consider  a sequential multi-class classification problem without the rejection option.

\subsection{Main Contribution}
Our contribution in the paper can be summarized as follows. In this paper, we extend the statistical classification problem with empirically observed statistics to the case when the decision maker observes the test sequence sequentially. First, we consider the binary classification problem and propose a test for the sequential setting. We analyse the performance of this test in terms of type-I and type-II error exponents (Theorem \ref{thm:main_b} and Corollary \ref{corr:achiev-bin}). Then, we show that this test outperforms   Gutman's test~\cite{gutman1989asymptotically} in terms of Bayesian error exponent (Theorem \ref{thm:comp_gut}). Furthermore, we generalize the problem setup to the multi-class classification. For this case, we describe an achievable scheme and provide a characterization of its error exponents (Theorem \ref{thm:main_multi} and Corollary \ref{corr:achiev-multi}). As a consequence of our results, we show that our test achieves the same performance as that of Gutman's but our test is arguably simpler as it does not consist of the rejection option (Theorem \ref{thm:gut_comp_multi}).  
\subsection{Paper Outline}
The remainder of the paper is organized as follows. Section \ref{sec:binary_case}
describes the problem setup and summarizes the main results for the binary classification. In Section \ref{sec:mutli} we extend the problem to the multi-class classification problem and presents our main results.
Sections \ref{sub:proof_main_b} and \ref{sec:proof_multi} are devoted to the proofs of the results provided in Section \ref{sec:binary_case} and \ref{sec:mutli}.

\subsection{Notations}
For each $m \in \mathbb{N}$, let $\left[m\right]\triangleq\left\lbrace 1, \hdots,m\right\rbrace$.  The set of all discrete distributions on alphabet $\mathcal{X}$ is denoted as $\mathcal{P}\left(\mathcal{X}\right)$. We use upper and lower letters to denote random variables and their realizations, respectively. For a vector of length $n$, we use the notation $x^n = \left(x_1, x_2, \hdots, x_n\right)\in \mathcal{X}^n$. Given
a vector $x^n = \left(x_1, x_2, \hdots, x_n\right)\in \mathcal{X}^n$, the type or empirical distribution is defined as
\begin{equation*}
\ED{x^n}\left(a\right)\triangleq \frac{1}{n} \sum_{i=1}^{n} \mathbbm{1}\{x_i=a\},\quad\forall\, a\in \mathcal{X},
\end{equation*}
 where $\mathbbm{1}\{\cdot\}$ denotes the indicator function. Also $\mathcal{T}_{n}$ represents the set of types with denominator $n$. The set of all sequences of length $n$ with type $Q$
is denoted by $\Gamma_Q^n$ (we sometimes omit $n$ if it is clear from the context). In addition, we use $\mathbb{E}\left[\cdot\right]$ to denote expectation, and, when not
clear from context, we use a subscript to indicate the distribution
with respect to which the expectation is being taken; e.g., $\mathbb{E}_{Q}\left[\cdot\right]$
denotes expectation with respect to the distribution $Q$.  Other notation concerning the method of types follows \cite[Chapter 11]{cover} and \cite{csiszar1998method}.
If $P$ is a distribution on $\mathcal{X}$ then $P^n$ is the $n$-fold i.i.d.\ product measure on $\mathcal{X}^n$, i.e., 
\begin{equation*}
P^n\left(x^n\right)= \prod_{i=1}^{n} P\left(x_i\right),\quad\forall\, x^n\in\calX^n.
\end{equation*}

The notion $a_n \doteq b_n$ means that $\frac{1}{n}\log \frac{a_n}{b_n} \to 0$ as $n \to \infty$. Similarly we can define $\dotleq$ and $\dotgeq$.
For other information-theoretic
notations we use the standard definitions, see e.g., \cite{cover}. Also, for a function $f: \mathbb{N}\rightarrow \mathbb{R}$, we say that $g(n) = O(f(n))$ if $\limsup_{n\to\infty}|f(n)/g(n)|<\infty$, $g(n)=o(f(n))$ if $\lim_{n\to\infty} |f(n)/g(n)|=0$, and $g(n) = \omega(f(n))$ if $\liminf_{n\to\infty} |f(n)/g(n)|=\infty$. Hence, for example, $o(1)$ denotes a vanishing sequence.
\section{Binary Sequential Classification }
\label{sec:binary_case}
\subsection{Problem Statement and Existing Results}
We assume that a decision maker has two training sequences of length $N$. The first and second training sequences are generated in an \iid manner according to $P_1\in\calP(\calX)$ and $P_2 \in\calP(\calX)$ respectively. The underlying distributions $(P_1, P_2)$ are \textit{unknown} but \textit{fixed} (i.e., remain unchanged throughout). 
The training sequences are denoted as $X_1^N \in \mathcal{X}^N$ and $X_2^N \in \mathcal{X}^N$. We fix a certain distribution $P_{i^*}$ where $i^*\in \{1,2\}$ which is unknown to the decision maker. Then, at each time $n \in \mathbb{N}$, a \textit{test sample} $Y_n\in \mathcal{X}$ as generated from   $P_{i^*}$  and $Y_n$ is given to the decision maker. The objective of the decision maker is to decide between the following two hypotheses:
\begin{itemize}
\item \textbf{$H_1$}: The test sequence up to the current time $\{Y_k\}_{k=1}^n$ (which is generated \iid according to $P_{i^*}$) and the first training sequence $X_1^N$ are generated according to the same distribution.
\item \textbf{$H_2$}: The test sequence up to the current time $\{Y_k\}_{k=1}^n$ and the second training sequence $X_2^N$ are generated according to the same distribution.
\end{itemize}
To achieve this goal, the decision maker at each time $n$ can take three actions:
\begin{enumerate}
\item Stop drawing a new test sample and declare the test sequence and the first training sequence are generated according to the same distribution.
\item Stop drawing a new test sample and declare the test sequence and the second training sequence are generated according to the same distribution.
\item Continue to draw  a new test sample from $P_{i^*}$.
\end{enumerate}
In contrast to sequential hypothesis testing \cite[Section 15.3]{yp_lec} where the two distributions are known, in this setup, the decision maker does not know either of the distributions. Instead, the only information decision maker has about $P_1$ and $P_2$ is through the two training sequences $X_1^N$ and $X_2^N$ 
generated in an i.i.d.\ fashion according to $P_1$ and $P_2$ respectively. Moreover, the problem considered here is different from  \cite{ziv1988classification,gutman1989asymptotically,zhousecond,unnikrishnan2015asymptotically,deanonym}
where the classification is studied for the cases that the length of the test sequence is {\em fixed} prior to the decision making. In our setup, we let the length of the test sequence be {\em random}. In fact, the total number of samples is a {\em stopping time} determined by the decision maker's action, i.e., this is a {\em variable-length setting}. Next, we provide
a precise formulation of this problem. We begin with the definition of the test for the aforementioned setup.
\begin{definition}[Test]  \label{def:test} A {\em test} is a pair $\Phi=\left(T,d\right)$ where
\begin{itemize}
\item The integer-valued random variable $T\in \mathbb{N}$ is a stopping time with respect to the filtration $\mathcal{F}_n=\sigma\{X_1^N,X_2^N,Y_1,\hdots,Y_n\}$ generated by the training samples and the test samples up to time $n$.
\item The map $d: \left(X_1^N,X_2^N,Y^T\right)\to \{H_1,H_2\}$ is a $\calF_T$-measurable decision rule.
\end{itemize}
\end{definition}
\begin{definition}[Type-I and Type-II Error Probabilities] 
 \label{def:error_prob}
For a test $\Phi=\left(T,d\right)$, the {\em type-I} and {\em type-II error probabilities} are defined as
\begin{align*}
\pe_{i}\left(\Phi\right)&=\mathrm{P}_{i}\left(d\left(X_1^N, X_2^N, Y^T\right)\neq H_i\right)
\end{align*}
for $i\in\{1,2\}$  respectively. Here $\probm_i$ denote the probability distribution under $H_i$. Also, note that for any $S \subseteq  \mathcal{X}^N \times \mathcal{X}^N \times \mathcal{X}^n $
we have 
$\probm_i(S) = \sum_{(x_1^N,x_2^N,y^n)\in S}  P_1^N(X_1^N) P_2^N(X_2^N) P_i^n(y^n)$. Also, $\mathbb{E}_i$ denotes the expectation under $\probm_i$.
\end{definition}
\begin{definition}[Error Exponents]\label{def:error_exp} For a test $\Phi=\left(T,d\right)$ such that $\pe_{i}\left(\Phi\right) \to 0$ as $N \to \infty$ and for $i\in\{1,2\}$, we define the type-$i$ error exponent as
\begin{align*}
\textsf{e}_i\left(\Phi\right) &= \liminf_{N \to \infty} \frac{-\log \pe_{i}\left(\Phi\right)}{\mathbb{E}_{i}\left[T\right]}
\end{align*}
where $\mathbb{E}_{i}\left[T\right]$ represents the expected value of the stopping time under hypothesis $H_i$.
\end{definition}
\begin{remark}
Note that the error event, i.e., $d\left(X_1^N, X_2^N, Y^T\right)\neq H_i$, and the random variable $T$ depend  on $N$. Furthermore, $\mathbb{E}_{i}\left[T\right]$ indicates the average number of the test samples under $H_i$ before the decision is made. 
\end{remark}
Gutman \cite{gutman1989asymptotically} considers the setup in which the decision maker has a test sequence $Y^n$ of \textit{fixed length $n$} which is independently generated from $X_1^N$ and $X_2^N$. 
Note as $N \to \infty$, $n$ also diverges but we have $\lim_{N,n\to\infty}\frac{N}{n}=\alpha$. For example, we can think always think of $N=\lfloor n\alpha\rfloor$.
To present Gutman's results, we need the following definition.
\begin{definition}[Generalized Jensen-Shannon  (GJS) Divergence]
Given $\alpha \in \mathbb{R}_{+}$ and $\left(P_1, P_2\right) \in \mathcal{P}({\mathcal{X}})^2$, the {\em generalized Jensen-Shannon (GJS) divergence} is defined as
\begin{equation}\label{eq:gjs_def}
 \GJS{P_1}{P_2}{\alpha} = \alpha \KL{P_1}{P_\alpha}+ \KL{P_2}{P_\alpha} .
\end{equation}
where $P_\alpha=\frac{\alpha P_1+P_2}{1+\alpha}$.
\end{definition} 
Theorem \ref{thm:gut1} summarizes Gutman's main results concerning with achievable error exponents and the converse results for the binary classification task using non-adaptive tests.
\begin{thm}(Gutman~\cite[Thm. 1]{gutman1989asymptotically}) \label{thm:gut1}
Let $\frac{N}{n}=\alpha$ and $\lambda \in \mathbb{R}_{+}$. Then Gutman's decision rule
\begin{equation} \label{eq:gutman_test_1}
\Phi_\text{GUT}(\lambda,\alpha)=
\begin{cases} 
  H_1 & \text{if } \GJS{\ED{X_1^N}}{\ED{Y^n}}{\alpha} \leq \lambda, \\
   H_2       & \text{if } \GJS{\ED{X_1^N}}{\ED{Y^n}}{\alpha} > \lambda ,
  \end{cases}
\end{equation}
has the following type-I and type-II error exponents
\begin{align}
\textsf{e}_1\left(\Phi_{\text{GUT}}(\lambda,\alpha)\right)&=\liminf_{N \to \infty} \frac{-\log \pe_{1}\left(\Phi_{\text{GUT}}(\lambda,\alpha)\right)}{N/\alpha} \geq \lambda, \label{eq:typeI_gut_basic}\\
\textsf{e}_2\left(\Phi_{\text{GUT}}(\lambda,\alpha)\right)&=\liminf_{N \to \infty} \frac{-\log \pe_{2}\left(\Phi_{\text{GUT}}(\lambda,\alpha)\right)}{N/\alpha} > F\left(\alpha, \lambda\right), \label{eq:typeII_gut_basic}
\end{align}
where 
\begin{equation}\label{eq:opt_gut}
\begin{aligned}
F(\alpha, \lambda)\triangleq & \underset{\left(Q_1, Q_2\right)\in \mathcal{P}\left(\mathcal{X}\right)^2}{\text{min}}
& & \alpha \KL{Q_1}{P_1}+ \KL{Q_2}{P_2} \\
& \text{subject to}
& & \GJS{Q_1}{Q_2}{\alpha}\leq \lambda.
\end{aligned}
\end{equation}
Also, Gutman's decision rule is optimal in the sense that among all non-adaptive decision rules $\Phi$, satisfying   $\textsf{e}_1\left(\Phi\right)\geq \lambda$  for all pairs of 
 distinct distributions $(P_1, P_2)\in \mathcal{P}\left(\mathcal{X}\right)^2$, one has $\textsf{e}_2\left(\Phi_{GUT}(\lambda,\alpha)\right)\geq \textsf{e}_2\left(\Phi\right)$. 
\end{thm}
\begin{remark}
\label{remark:intuition-gutman}
The intuition for Gutman's test in \eqref{eq:gutman_test_1} and the bounds in \eqref{eq:typeI_gut_basic} and \eqref{eq:typeII_gut_basic} are as follows. The rule in \eqref{eq:gutman_test_1} posits that we should choose $H_1$ if the type of the first set of training samples $X_1^N$ and the test sequence $Y^n$ are ``close''. The appropriate measure of closeness in this scenario is the GJS with parameter $\alpha$ because the GJS arises naturally as the exponent when one uses Sanov's theorem to establish the exponential rates of decay of the error probabilities and carefully takes into account the different lengths of the training and test sequences (see Lemma \ref{lem:def_gjs} and Lemma \ref{lem:error_same}). The bounds in \eqref{eq:typeI_gut_basic} and \eqref{eq:typeII_gut_basic} are natural consequences in view of Sanov's theorem. In fact, similar to the Neyman-Pearson rule, Gutman's test is optimal (cf.\ \cite{zhousecond}).
\end{remark}
\subsection{Main Results}
We now describe our sequential classification test. Fix a threshold parameter $\gamma \in \mathbb{R}_{+}$. The proposed test for the sequential classification is $\Phi_{\text{seq}}(\gamma)=\left(T_{\text{seq}},d_{\text{seq}}\right)$ where $T_{\text{seq}}$ and $d_{\text{seq}}$ are defined as
\begin{equation}\label{eq:def_stoptime}
\begin{aligned} 
T_{\text{seq}} = \inf \bigg\{ n\geq 1 : &\exists\,   i \in \{1,2\}
~ \text{such that} ~ \\   
&n \GJS{\ED{X_i^N}}{\ED{Y^n}}{\frac{N}{n}} \geq \gamma N  \bigg\} \wedge N^2,
\end{aligned}
\end{equation}
and
\begin{equation}\label{eq:rule}
d_{\text{seq}} = \begin{cases} H_1 & \text{if $T_{\text{seq}}\GJS{\ED{X_2^N}}{\ED{Y^{T_{\text{seq}}}}}{\frac{N}{T_{\text{seq}}}}  \geq \gamma N$} \\ 
 H_2 & \text{if $T_{\text{seq}}\GJS{\ED{X_1^N}}{\ED{Y^{T_{\text{seq}}}}}{\frac{N}{T_{\text{seq}}}}  \geq \gamma N$}  \end{cases},
\end{equation}
respectively. In (\ref{eq:def_stoptime}), $\wedge$ denotes the pairwise minimum operation, i.e., $a \wedge b = \min (a,b)$.

We can view the the decision rule in (\ref{eq:rule}) as assigning a \textit{score} at each time $n \in \mathbb{N}$, i.e.,  $\text{score}_i[n]= n \GJS{\ED{X_i^N}}{\ED{Y^n}}{\frac{N}{n}}$, to each class. At the first time that the score of one of the class exceeds the threshold, i.e., $\gamma N$, the decision maker outputs the class with the least score.  As an illustrative example, Figure \ref{fig:emp_stp} shows a realization of $\Phi_{\text{seq}}(\gamma)$ for two ternary source distributions $P_1=[0.1,0.7,0.2]$ and $P_2=[0.05,0.55,0.4]$. The threshold is  $\gamma=0.02$. The test sequence is drawn from $P_2$ and the length of the training sequence is $N=400$. Note that the stopping time defined in~(\ref{eq:def_stoptime}) is $T_{\text{seq}}=141$ since at $n=141$,  the score for class $1$, i.e., $n\GJS{\ED{X_1^N}}{\ED{Y^n}}{\frac{N}{n}}$, exceeds the threshold $\gamma$. Therefore, based on (\ref{eq:rule}), we declare $H_2$ as the final decision. 
\begin{remark}
Note that in the definition of $T_{\text{seq}}$ in (\ref{eq:def_stoptime}), $N^2$ can be replaced by any function $h\left(\cdot\right): \mathbb{N}\rightarrow \mathbb{N}$ with the following properties: 1) $h(N)=\omega\big(N^{\frac{3}{2}}\big)$ and, 2) $\frac{1}{N}{\log h(N)}=o\left(1\right)$. It can be verified that $h(N)=N^2$ satisfies the aforementioned conditions. 
\end{remark}
Before, presenting our main result on   binary classification, we need the following definition.
\begin{definition}
The {\em Chernoff information} between two probability mass functions  $P\in \mathcal{P}\left(\mathcal{X}\right)$ and $Q\in \mathcal{P}\left(\mathcal{X}\right)$ is defined as
\begin{equation}
C\left(P, Q\right)\triangleq - \min _{\eta \in \left[0,1\right]} \log \sum_{x \in \mathcal{X}} P\left(x\right)^{\eta} Q\left(x\right)^{1-\eta}. \label{eqn:chernoff}
\end{equation}
\end{definition}
In the next theorem, we present the main result of this section which is on the properties of test $\Phi_{\text{seq}}(\gamma)$ and the achievable type-I and type-II error exponents of the proposed test for the binary classification problem. The proof of Theorem \ref{thm:main_b} is provided in Section \ref{sub:proof_main_b}.
\begin{thm}\label{thm:main_b}
Fix pair $\left(P_1,P_2\right)\in \mathcal{P}\left(\mathcal{X}\right)^2$ and $\gamma \in  ( 0, C\left(P_1,P_2\right)]$. Define
 $\beta^{\star}_\gamma \in\mathbb{R}_+$ to be the solution of
\begin{equation}\label{eq:fixed1}
\mathrm{GJS}\left(P_2,P_1,\beta^{\star}_\gamma\right)=\gamma \beta^{\star}_\gamma.
\end{equation}
Similarly, define  $\theta^{\star}_\gamma\in\mathbb{R}_+$ to be the solution of
\begin{equation}\label{eq:fixed2}
\mathrm{GJS}\left(P_1,P_2,\theta^{\star}_\gamma\right)=\gamma \theta^{\star}_\gamma.
\end{equation} 
Then, the proposed test has the following properties: 
\begin{itemize}
\item $\pe_{1}\left(\Phi_{\text{seq}}(\gamma)\right) \dotleq  \exp\left(-N \gamma\right)$
\item $\mathbb{E}_1\left[T_{\text{seq}}\right]=\frac{N}{\beta^{\star}_\gamma}\left(1+o\left(1\right)\right)$
\item $\pe_{2}\left(\Phi_{\text{seq}}(\gamma)\right) \dotleq  \exp\left(-N \gamma\right)$
\item $\mathbb{E}_2\left[T_{\text{seq}}\right]=\frac{N}{\theta^{\star}_\gamma}\left(1+o\left(1\right)\right)$
\end{itemize}
\end{thm}
The following corollary summarizes our results on the expected value of the stopping time. 
\begin{corollary}\label{corr:achiev-bin}
The achievable Type-I and Type-II error exponents of the proposed test are given by
\begin{align}
 \textsf{e}_1\left(\Phi_{\text{seq}}(\gamma)\right) &  \geq \mathrm{GJS}\left(P_2,P_1,\beta^{\star}_\gamma\right)\label{eq:typeI_prop},\\
 \textsf{e}_2\left(\Phi_{\text{seq}}(\gamma)\right) & \geq \mathrm{GJS}\left(P_1,P_2,\theta^{\star}_\gamma\right)\label{eq:typeII_prop},
\end{align}
where $\alpha^{\star}_\gamma$ and $\beta^{\star}_\gamma$ are given by in (\ref{eq:fixed1}) and (\ref{eq:fixed2}) respectively.
\end{corollary}
\begin{figure}[t]
\centering
\includegraphics[scale=0.45]{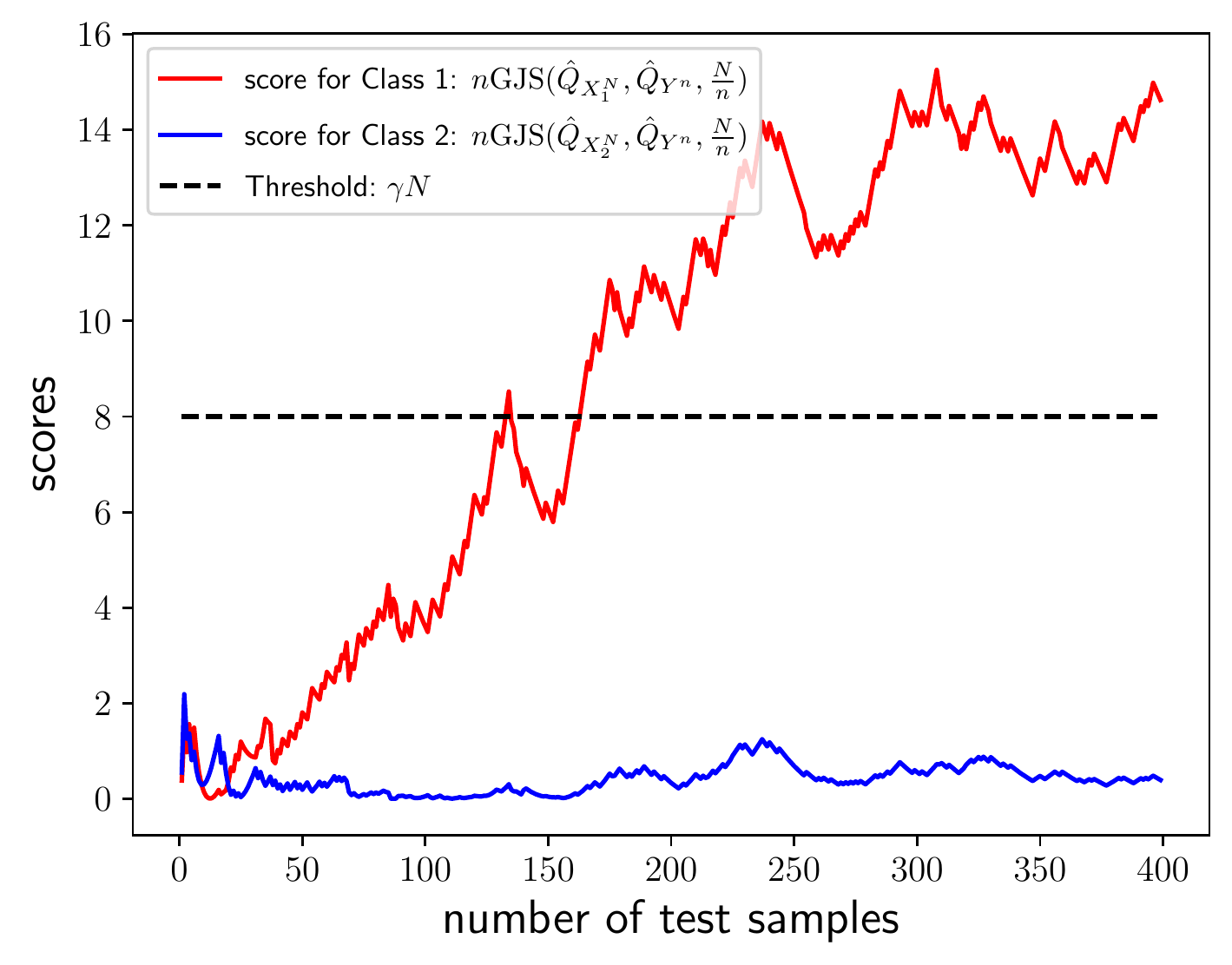}
\caption{A realization of the sequential test $\Phi_{\text{seq}}(\gamma)=\left(T_{\text{seq}},d_{\text{seq}}\right)$ in (\ref{eq:def_stoptime}) and (\ref{eq:rule}).}
\label{fig:emp_stp}
\end{figure}
\begin{remark}
\label{rem:sprt-our-bin}
In this remark we provide a {\em partial} converse bound for our results in Corollary \ref{corr:achiev-bin}. It is easy to observe that the performance of the SPRT provides an upper bound on the performance of our test. Therefore, combining  \cite[Thm. 15.3]{yp_lec} with (\ref{eq:typeI_prop}) and  (\ref{eq:typeII_prop}), we obtain the following upper bound
\begin{align*}
&\textsf{e}_1\left(\Phi_{\text{seq}}(\gamma)\right) \textsf{e}_2\left(\Phi_{\text{seq}}(\gamma)\right)\leq \KL{P_1}{P_2} \KL{P_2}{P_1},
\end{align*}
and lower bound
\begin{align*}
&\textsf{e}_1\left(\Phi_{\text{seq}}(\gamma)\right) \textsf{e}_2\left(\Phi_{\text{seq}}(\gamma)\right)\geq \GJS{P_2}{P_1}{\beta^{\star}_\gamma} \GJS{P_1}{P_2}{\theta^{\star}_\gamma}.
\end{align*}
Indeed, if we let $\gamma \to 0$ which corresponds to the case that $\beta^{\star}_{\gamma} \to \infty $ and $\theta^{\star}_{\gamma} \to \infty$
we have that 
\begin{align*}
\lim_{\beta^{\star}_{\gamma},\theta^{\star}_{\gamma} \to \infty}\GJS{P_1}{P_2}{\theta^{\star}_\gamma} &\GJS{P_2}{P_1}{\beta^{\star}_\gamma} \\
&= \KL{P_2}{P_1} \KL{P_1}{P_2}
\end{align*}
This is because of our results in Lemma \ref{lem:g}.
This shows that our sequential scheme can recover the optimal  performance in the case that $\beta^{\star}_{\gamma} \to \infty $ and $\theta^{\star}_{\gamma} \to \infty$.
\end{remark}
\subsection{Comparison to Gutman's Scheme}
\label{sec:compare_binary}
In this subsection, we compare the proposed sequential test with the Gutman's fixed-length test.  We adopt a Bayesian approach in which we assign prior probabilities $\pi_1 \in (0,1)$ and $\pi_2 = 1-\pi_1$ to $H_1$ and $H_2$  respectively. In this case, the {\em overall average probability of error} is given by
\begin{equation}
\begin{aligned}
\pe(\Phi) =& \pi_1 \mathbb{P}\left(d\left(X_1^N, X_2^N, Y^T\right)\neq H_1\big| H_1\right)\\
 +&\pi_2 \mathbb{P}\left(d\left(X_1^N, X_2^N, Y^T\right)\neq H_2\big| H_2\right)
\end{aligned}
\end{equation}
We define the error exponent for the Bayesian scenario as
\begin{equation}\label{eq:bayesian_exponent}
\textsf{e}_{\text{Bayesian}}^\pi\left(\Phi\right)\triangleq \liminf_{N\to \infty}\frac{- \log \pe(\Phi) }{N}.
\end{equation}
To make a fair comparison, let us assume each of the  schemes, sequential and Gutman, has  two training sequences of length $N$.  
For Gutman's test, as usual we let $\frac{N}{n}=\alpha$ and we consider the following variation 
\begin{equation} \label{eq:gutman_test_tweeked}
\Phi_\text{GUT}(\lambda,\alpha)=
\begin{cases} 
  H_1 & \text{if } \GJS{\ED{X_1^N}}{\ED{Y^n}}{\alpha} \leq \lambda \alpha, \\
   H_2       & \text{if } \GJS{\ED{X_1^N}}{\ED{Y^n}}{\alpha} \geq \lambda \alpha ,
  \end{cases},
\end{equation}
in which $\lambda$ in \eqref{eq:gutman_test_1} is replaced by $\lambda\alpha$ here without loss of generality. 
For $\Phi_\text{GUT}(\lambda,\alpha)$, it can be readily shown that 
\begin{align} \label{eq:bayes_1}
\textsf{e}_{\text{Bayesian}}^{\pi}\left( \Phi_{\text{GUT}}(\lambda^{\star},\alpha\right))=\max\limits_{\lambda\geq 0} \min\left\lbrace \lambda,F_1\left(\alpha,\lambda\right) \right\rbrace
\end{align}
where 
\begin{equation}\label{eq:opt_gut_N}
\begin{aligned}
F_1\left(\alpha,\lambda\right)\triangleq & \underset{\left(Q_1, Q_2\right)\in \mathcal{P}\left(\mathcal{X}\right)^2}{\text{min}}
& & \KL{Q_1}{P_1}+ \frac{1}{\alpha} \KL{Q_2}{P_2} \\
& \text{subject to}
& & \frac{1}{\alpha}\mathrm{GJS}\left(Q_1, Q_2, \alpha \right)\leq \lambda.
\end{aligned}
\end{equation}
In the next lemma, we study the impact of $\alpha$ on the {\em Bayesian error exponent} $\textsf{e}_{\text{Bayesian}}\left(\Phi_\text{GUT}\right)$. 
\begin{lemma}\label{lem:dec_bayes}
The function $\alpha\mapsto \textsf{e}_{\text{Bayesian}}^{\pi}\left( \Phi_{\text{GUT}}(\lambda^{\star},\alpha\right))$ is decreasing.%
\begin{proof}
It is straightforward to show that the objective function of (\ref{eq:opt_gut_N}) is decreasing in $\alpha$. Also, because $\frac{1}{\alpha}\GJS{Q_1}{Q_2}{\alpha}$ is decreasing in $\alpha$, we conclude that the feasible set is enlarged as  $\alpha$ increases.
\end{proof}
\end{lemma}
Gutman's test is designed for the case that a fixed-length test sequence is provided to the decision maker. On the other hand, from Theorem \ref{thm:main_b} we know that in the sequential test the average number of the test samples under $H_1$ and $H_2$ are different and given approximately  by ${N}/{\beta^{\star}_\gamma}$ and ${N}/{\theta^{\star}_\gamma}$ respectively. %
In light of Lemma \ref{lem:dec_bayes}, we will assume, for the sake of comparisons (between Gutman's test and ours), that Gutman's test is provided with ${N}/{\min \{\theta^{\star}_\gamma,\beta^{\star}_\gamma\}}$ test samples, i.e.,  the (deterministic) number of samples in the testing sequence used by the Gutman's test is equal to the largest expected sample complexity of the sequential
test under any of the two hypotheses. We  assert that under this assumption,   
 it is fair to compare the sequential   and Gutman's tests. The next theorem presents our results concerning the comparison between these two tests.
\begin{thm} \label{thm:comp_gut}
Consider the scenario in which  ${N}/{\min \{\theta^{\star}_\gamma,\beta^{\star}_\gamma\}}$ samples are available to be used in Gutman's test. Then, achievable Bayesian error exponent of the sequential scheme is {\em strictly greater} than the Bayesian error exponent of Gutman's test.
\begin{proof}
The proof is provided in Appendix~\ref{subsec:proof_of_comp}.
\end{proof}
\end{thm}
We showed that $\textsf{e}_{\text{Bayesian}}^{\pi}\left( \Phi_{\text{seq}}(\gamma)\right)\geq \gamma$, and we know $\gamma \in (0, C(P_1,P_2)]$. In the next Corollary we provide the maximum achievable Bayesian error exponent of $\Phi_{\text{seq}}(\gamma)$.
\begin{corollary}\label{corr:max-bayes-exp-bin}
The maximum achievable Bayesian error exponent of the sequential scheme  $\Phi_{\text{seq}}(\gamma)$  is 
$ C\left(P_1,P_2\right)$.
\end{corollary}
In Figure \ref{fig:comp2}, we provide a numerical example to quantitatively illustrate the gain of our proposed test versus that of the Gutman. We plot an {\em achievable Bayesian error exponent} based on our sequential scheme compared to the Bayesian error exponent of Gutman's test versus $\min \{\theta^{\star}_\gamma,\beta^{\star}_\gamma\}$. We consider a ternary alphabet $\mathcal{X}=\{1,2,3\}$. In Fig.~\ref{fig:comp2}, we set $P_1 = \left[0.1, 0.3, 0.6\right]$ and $P_2=\left[0.45, 0.45, 0.1\right]$. As the problem in (\ref{eq:opt_gut_N}) is a convex problem, we used \textsf{CVXPY} package \cite{cvxpy} to perform the optimization. This example shows that sequential test significantly improves the Bayesian error exponent over Gutman's fixed-length test.

\begin{figure}[t]
\centering
\includegraphics[scale=0.450]{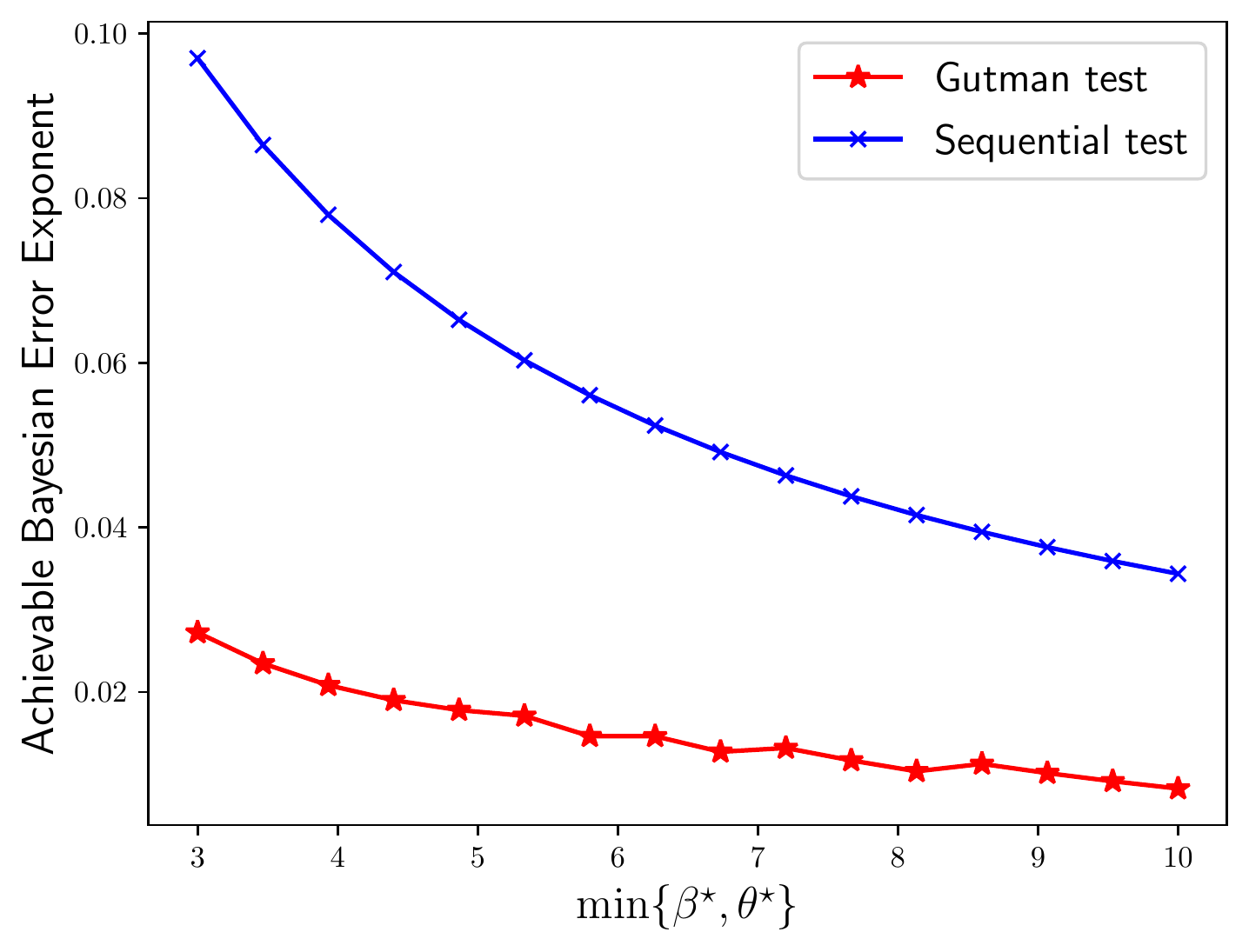}
\caption{Comparison of Gutman's test and the sequential test in terms the Bayesian error exponent for $P_1 = \left[0.1, 0.3, 0.6\right]$ and $P_2=\left[0.45, 0.45, 0.1\right]$. }
\label{fig:comp2}
\end{figure}

\section{Sequential Classification: Multi-Class Classification Problem}
\label{sec:mutli}
In this section, we extend the binary classification setup  to the scenario in which we have $M>2$ classes. 
\subsection{Problem Statement}
In the classification problem with $M$ classes, the decision maker has access to $M$ length-$N$ training sequences denoted by $\left\lbrace X_i^N\right\rbrace_{i=1}^{M}$. Each training sequence is  generated in an \iid manner according to  one of $M$  unknown distributions $\left(P_1,\hdots,P_M\right) \in \mathcal{P}\left(\mathcal{X}\right)^{M}$.  We fix a certain distribution $P_{i^*}$ where $i^*\in \{1,\ldots, M\}$.  At each time $n \in \mathbb{N}$, a test sample, denoted by $Y_n$, is generated according to $P_{i^*}$ and is given to the decision maker. The decision maker is tasked to {\em classify} the test sequence $\{Y_k\}_{k=1}^n$, i.e., assign it a label from the set $\{1,\ldots, M\}$. More formally, the decision maker has to decide between the following $M$ hypotheses:
\begin{itemize}
\item $H_i$ where $i \in \{1,\hdots,M\}$: The test sequence $\{Y_k\}_{k=1}^n$  (which is generated \iid according to $P_{i^*}$) and the $i^{\mathrm{th}}$ training sequence $X_i^N$ are being generated according to the same distribution.
\end{itemize}  
For the described setup, the test, the error probabilities, and the error exponents can be defined analogously to Definitions~\ref{def:test},~\ref{def:error_prob}, and~\ref{def:error_exp}, respectively. However, here the stopping time $T$ is now adapted to  the filtration $\mathcal{F}_n=\sigma \big\{ \{X_i^N \}_{i=1}^{M} ,Y_1\hdots,Y_n \big\}$, and the terminal decision rule is a function of $\big( \{X_i^N\}_{i=1}^{M},Y^{T}\big)$.

The multiclass classification problem has been considered from information-theoretic perspectives in \cite{gutman1989asymptotically,zhousecond,ziv1988classification,unnikrishnan2015asymptotically,deanonym}. There are two main aspects that distinguish our work with previous studies. First, the lengths of the test sequence is fixed in  \cite{gutman1989asymptotically,zhousecond,ziv1988classification,unnikrishnan2015asymptotically,deanonym}; however, we let the length of test sequence be random. Second, the $M$-class classification problem is often studied  with the rejection option in the literature; our setup does not include the rejection option. More precisely, in \cite{gutman1989asymptotically,zhousecond,ziv1988classification,unnikrishnan2015asymptotically,deanonym} the decision maker has the following $M+1$ hypotheses:
\begin{itemize}
\item $H_i$ for each $i \in \left[M\right]$: The test sequence  $Y^n$ and the $i^{\mathrm{th}}$ training sequence $X_i^N$ are generated according to the same distribution.
\item $H_{\mathrm{r}}$: The test sequence $Y^n$ is generated according to a distribution different from those in which the training sequences are
generated from.
\end{itemize}
Provided that the length of the test sequence $n$ is fixed, in this framework, the error probabilities and the rejection probability are defined as
\begin{align}
\pe_i(\Phi)  &=\probm_i\left(d\left( \{ X_i^N\}_{i=1}^{M}, Y^n\right)\notin \{H_i,H_{\mathrm{r}}\} \right)\\
 \mathrm{P}_i^{\mathrm{rej}}(\Phi)&=\probm_i\left(d\left(\{ X_i^N\}_{i=1}^{M}, Y^n \right)= H_{\mathrm{r}}\right) \label{eq:rejection_prob}
\end{align}
for $i\in \left[M\right]$. Here note that we are considering {\em realizable} case in which we know that the test sequence and one of the training sequences are generated according to the same distribution. In (\ref{eq:rejection_prob}), $\mathrm{P}_i^{\mathrm{rej}}(\Phi)$ denotes the probability that the decision maker declares $H_{\mathrm{r}}$ as the terminal decision (the rejection option is taken here).  The main result for the setup with fixed-length test sequence and the rejection option is by Gutman \cite{gutman1989asymptotically}. Note as $N \to \infty$, $n$ also diverges but we have $\lim_{N,n\to\infty}\frac{N}{n}=\alpha$. For example, we can always think of $N=\lfloor n\alpha\rfloor$. In \cite{gutman1989asymptotically} the following questions are addressed: What is the largest $\lambda\in\bbR$ for which 
\begin{enumerate}
\item for $i \in \left[M\right]$, we have
\begin{equation}
\liminf_{n \to \infty} \frac{-\log \pe_i(\Phi_{\text{GUT}}^{(M)}(\lambda,\alpha))}{n} \geq \lambda,
\end{equation}
where $\Phi_{\text{GUT}}^{(M)}(\lambda,\alpha)$ denotes the Gutman's test; and 
\item the rejection probability tends to zero as $n$ goes to infinity?
\end{enumerate}
The next theorem provides an answer to the aforementioned question.
\begin{thm}(Gutman~\cite[Thms.~2 \& 3]{gutman1989asymptotically})
\label{thm:gut_multi}
Assume $\frac{N}{n}=\alpha$. Then, the maximum $\lambda$ satisfying both conditions mentioned above is 
\begin{equation}
\widehat{\lambda} = \min\limits_{i,j \in \left[M\right],i\neq j} \GJS{P_i}{P_j}{\alpha}.
\end{equation}
Moreover, if $\lambda > \widehat{\lambda}$, the rejection probability goes to \textit{one} as $n$ goes to infinity.
\end{thm}

\subsection{Main Results}
In this section, we present our proposed test which does not utilize the rejection option. 
Let $\gamma \in \mathbb{R}_{+}$ be a fixed threshold for the test. For $n\in \mathbb{N}$, define the set 
\begin{equation}
\begin{aligned}
\Psi_n \triangleq \bigg\{ i\in \{1,\hdots,M\} :  &\exists\, 1 \leq k \leq n \ \text{such that} \\ 
& k \GJS{\ED{X_i^N}}{\ED{Y^k}}{\frac{N}{k}}\geq \gamma N  \bigg\}.
\end{aligned}
\end{equation}
Then, the proposed stopping time is 
\begin{equation} \label{eq:def_stoptime_m}
T_{\text{seq}}^{(M)} \triangleq \inf \big\{ n\geq 1 :  | \Psi_n  | \geq M-1  \big\}\wedge N^2.
\end{equation}
Also, at time $T_{\text{seq}}^{(M)}$, the terminal decision rule is
\begin{equation} \label{eq:def_dec_rule_m}
d_{\text{seq}}^{(M)}\triangleq \left[M\right] \setminus \Psi_{T_{\text{seq}}^{(M)}}.
\end{equation}
The proposed test is denoted by $\Phi_{\text{seq}}^{(M)}(\gamma)=\big(T_{\text{seq}}^{(M)},d_{\text{seq}}^{(M)}\big)$. \par
Figure \ref{fig:emp_seq_M} shows a realization of $\Phi_{\text{seq}}^{(M)}(\gamma)$ for three ternary distributions $P_1=[0.1,0.7,0.2]$, $P_2=[0.4,0.5,0.1]$,  and $P_3=\left[0.3, 0.3, 0.4 \right]$, and $\gamma=0.03$ where the test sequence drawn from $P_2$ and the length of each training sequence is $N=300$. Note that the stopping time defined in (\ref{eq:def_stoptime_m}) is $T_{\text{seq}}^{(M)}=40$ since at $n=40$, we have $|\Psi_{73}|=2=M-1$. Then, according to~(\ref{eq:def_dec_rule_m}), the final decision rule is $H_2$.  We now recall that  $C\left(P_i,P_j\right)$,  defined in~\eqref{eqn:chernoff}, denotes the Chernoff information between $P_i$ and $P_j$. This will feature in the next theorem. 
\begin{thm}
\label{thm:main_multi}
Fix $\left(P_1,\hdots,P_M\right) \in \mathcal{P}\left(\mathcal{X}\right)^M$. Let $\mathcal{M}\triangleq \{\left(i,j\right) \in \left[M\right]^2,i\neq j  \}$.  Then, for any $\gamma \in \big[0, \min_{\left(i,j\right)\in \mathcal{M}} C\left(P_i,P_j\right)\big]$,
the proposed test achieves
\begin{itemize}
\item $\pe_i(\Phi^{(M)}_{\text{seq}}(\gamma)) \dotleq \exp\left(-N \gamma\right)$
\item $ \mathbb{E}_i\left[T^{(M)}_{\text{seq}} \right]=\frac{N}{\min_{j\in \left[M\right],j\neq i}\{\theta^{\star}_{i(j),\gamma}\}}\left(1+o\left(1\right)\right)$
\end{itemize}
where $\theta^{\star}_{i(j),\gamma}$ is given by the solution to the following equation
\begin{equation} \label{eq:beta_star_M}
\GJS{P_j}{P_i}{\theta^{\star}_{i(j),\gamma}} = \gamma \theta^{\star}_{i(j),\gamma} ,\qquad\forall\, (i,j) \in \mathcal{M}.
\end{equation}
\end{thm}
\begin{corollary}
\label{corr:achiev-multi}
The achievable error exponents of the proposed test under $H_i$ is given by
\begin{equation}
\begin{aligned}
\textsf{e}_i\left(\Phi^{(M)}_{\text{seq}}(\gamma)\right)&=\liminf\limits_{N \to \infty}\frac{-\log \pe_i(\Phi_{\text{seq}}^{(M)}(\gamma))}{\mathbb{E}_{i}\big[ T^{(M)}_{\text{seq}}\big]}\\
&\geq \min_{j\in \left[M\right],j\neq i}\GJS{P_j}{P_i}{\theta^{\star}_{i(j),\gamma}},
\end{aligned}
\end{equation}
for all $ i \in \left[M\right]$
\end{corollary}
\begin{figure}[t]
\centering
\includegraphics[scale=0.450]{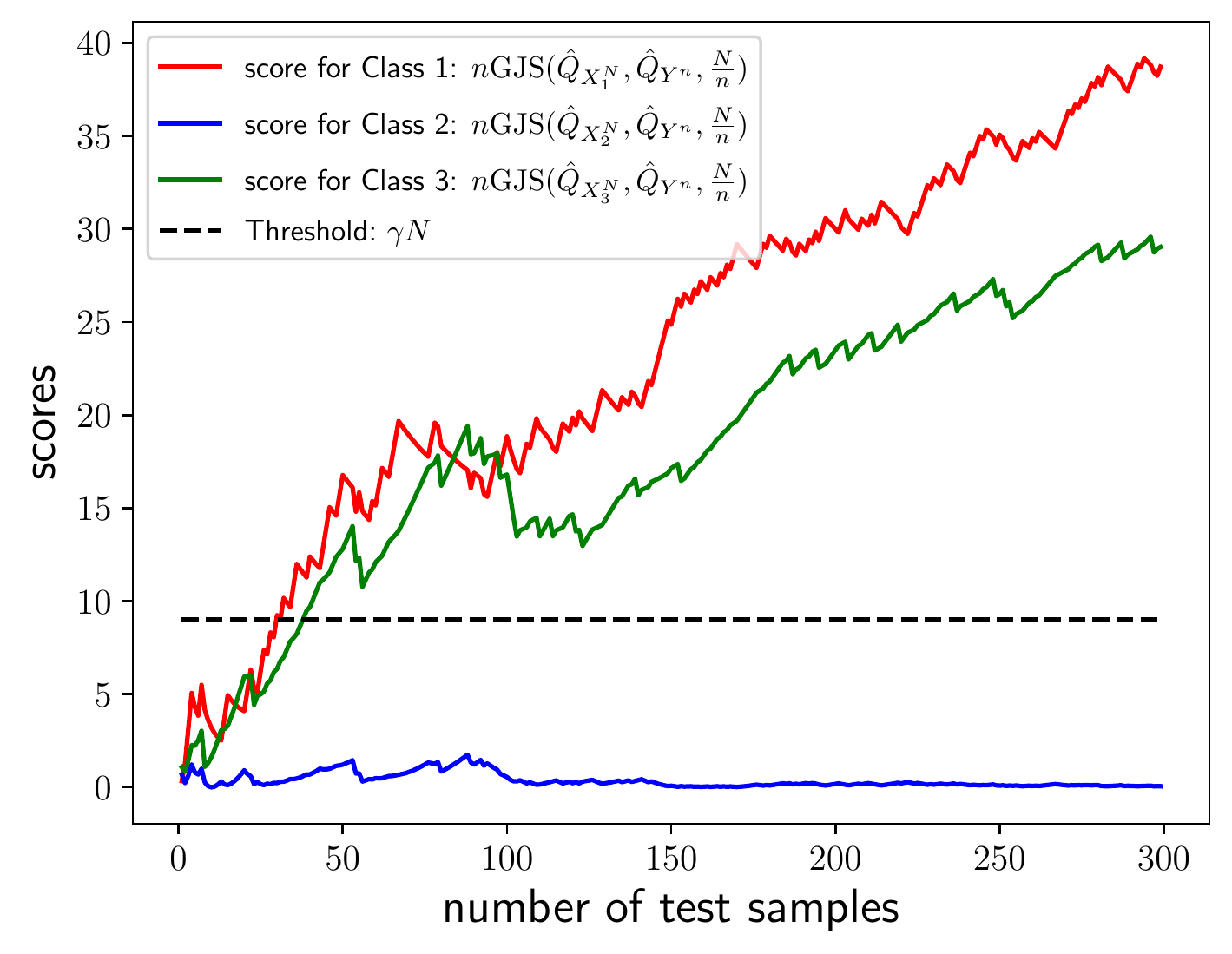}
\caption{A realization of the sequential test  $\Phi_{\text{seq}}^{(M)}$ for the multi-class classification. }
\label{fig:emp_seq_M}
\end{figure}
\subsection{Comparison to Gutman's test for the multiclass classification problem}
In this section, we compare the Gutman's test for the multiclass classification problem using the test $\Phi_{\text{seq}}^{(M)}$.  We adopt a Bayesian approach in which we assign prior probabilities $\pi_i , i \in [M]$ to hypotheses $H_i, i\in [M]$. In this case, the {\em average probability of error} is given by
\begin{equation}
\pe(\Phi) = \sum_{i=1}^{M} \pi_i \mathbb{P}\left(d\left( \{ X_i^N\}_{i=1}^{M}, Y^n\right)\neq H_i \big| H_i\right).
\end{equation}
Here, we are interested in the Bayesian error exponent defined in~(\ref{eq:bayesian_exponent}).
The main result by Gutman in Theorem  \ref{thm:gut_multi} can be restated in terms of the Bayesian error exponent as follows. The maximum $\lambda_{\text{Bayesian}}$ for which we have 

$$ \textsf{e}_{\text{Bayesian}}^\pi\left(\Phi_{\text{GUT}}^{(M)}(\lambda,\alpha)\right)\geq \lambda_{\text{Bayesian}},
$$

and the rejection probability defined in (\ref{eq:rejection_prob}) tends to zero as $N\to \infty$ is 
\begin{equation}
\label{eq:exp_bayesian_gut}
\lambda^{\star}_{\text{Bayesian}} = \min\limits_{(i,j) \in  \mathcal{M}}\frac{ \GJS{P_i}{P_j}{\alpha}}{\alpha}.
\end{equation}  
It can be shown that $\lambda^{\star}_{\text{Bayesian}}$ is a decreasing function of $\alpha$. We follow the same approach as in Section \ref{sec:compare_binary} where we compared  Gutman's scheme with our proposed test for the binary case. We argue that assigning $\alpha $ to be $\min_{(i,j) \in  \mathcal{M}} \theta^{\star}_{i(j)}$ where $\theta^{\star}_{i(j)}$, defined in (\ref{eq:beta_star_M}), ensures that the %
comparison of the achievable Bayesian error exponents of Gutman's scheme and the sequential test is fair. From Theorem \ref{thm:main_multi}, it is straightforward to show that 
\begin{equation}
\textsf{e}_{\text{Bayesian}}^\pi\left(\Phi_{\text{seq}}^{(M)}(\gamma)\right)\geq \gamma.
\end{equation}
Here, we claim that setting $\alpha=\min_{(i,j) \in  \mathcal{M}} \theta^{\star}_{i(j)}$ in (\ref{eq:exp_bayesian_gut}), we get 
\begin{align}
\lambda^{\star}_{\text{Bayesian}}&=\min\limits_{\left(l,k\right)\in \mathcal{M}}\frac{ \GJS{P_l}{P_k}{\min_{\left(i,j\right)\in \mathcal{M}} \theta^{\star}_{i(j),\gamma}}}{\min_{\left(i,j\right)\in \mathcal{M}} \theta^{\star}_{i(j),\gamma}} \nonumber \\
&=\gamma,
\end{align}
where the last step can be proved as follows: For all $ (l,k) \in \mathcal{M}$, we have 
\begin{equation}
\label{eq:bayes_proof_multi}
\GJS{P_l}{P_k}{\min_{\left(i,j\right)\in \mathcal{M}} \theta^{\star}_{i(j),\gamma} }\geq \gamma \min_{\left(i,j\right)\in \mathcal{M}} \theta^{\star}_{i(j),\gamma}.
\end{equation}
The reason for \eqref{eq:bayes_proof_multi} is as follows. Define $f_{l,k}\left(\theta\right) \triangleq \GJS{P_l}{P_k}{\theta} - \gamma \theta $. Note that $f_{l,k}(\theta^{\star}_{k(l),\gamma})=0$. %
Furthermore, for $\theta \leq \theta^{\star}_{k(l),\gamma}$ we have $f_{l,k}\left(\theta\right) \geq \gamma \theta$. Since $\min_{\left(i,j\right)\in \mathcal{M}} \theta^{\star}_{i(j),\gamma} \leq \theta^{\star}_{k(l),\gamma}$, we have (\ref{eq:bayes_proof_multi}) as required. 
Therefore, we showed that the Bayesian error exponents of Gutman's test and the sequential test are the same. However, recall that the sequential test achieves this performance {\em without the rejection option}. The next theorem summarizes our discussion in this section. 
\begin{thm}
\label{thm:gut_comp_multi}
The achievable Bayesian error exponent of the sequential test, i.e., $\Phi_{\text{seq}}^{(M)}(\gamma)$ is equal to that of the Gutman's $\Phi_{\text{GUT}}^{(M)}(\gamma)$. Gutman's test achieves this performance by introducing the rejection option, while the sequential test does not have this option. 
\end{thm}

\begin{corollary}
The maximum achievable Bayesian error exponent of the sequential scheme  $\Phi_{\text{seq}}^{(M)}(\gamma)$  is 
\begin{equation}
\min_{\left(i,j\right)\in \left[M\right]^2,i\neq j} C\left(P_i,P_j\right),
\end{equation}
where $C\left(P_i,P_j\right)$ is Chernoff information between $P_i$ and $P_j$.
\end{corollary}
\begin{remark}
\label{rem:reject-seq}
It is interesting to note that the possibility of removing the ``rejection region'' provided that sequential tests are allowed has been shown in other contexts. For instance, in \cite[Theorem~6]{lee2016} and \cite[Theorem~1]{hayashi2015} this phenomenon was shown in the context of block coding and streaming data transmission, respectively. However, to the best of our knowledge, none of the previous works consider scenarios in which the distributions are unknown or partially known. This works aims to formalize this idea the more practical statistical learning scenario in which one has partial, noisy information about the underlying distributions in the form of finite-length training samples.
\end{remark}
\section{Proofs of the Main Results}
\label{sec:proof_all_results}
\subsection{ Preliminary Lemmas}
\label{app:proof}
Subsection \ref{app:proof} is devoted to some preliminary lemmas that will be used in the sequel. We first start with Lemma \ref{lem:def_gjs} which provides a variational representation of the GJS divergence.
\begin{lemma}\label{lem:def_gjs}
Let $v^N$ and $w^n$ are two sequences with types $Q_1$ and $Q_2$ over the alphabet $\mathcal{X}$, respectively. We form the following optimization problem.
\begin{equation}\label{opt:def_gjs}
\begin{aligned}
& \underset{P}{\text{min}}
& & -\frac{1}{n}\log P^{n+N }\left(v^N, w^n\right) \\
& \text{subject to}
& &  P\in \mathcal{P}\left(\mathcal{X}\right).
\end{aligned}
\end{equation}
Then, the optimal value of (\ref{opt:def_gjs}) lower  bounded by
$ \GJS{Q_1}{Q_2}{\frac{N}{n}}$. Also, the optimal solution is given by $P^{\star}=\frac{\frac{N}{n}Q_1 + Q_2}{1+\frac{N}{n}}$.
\begin{proof}
We begin the proof by rewriting   
\begin{align}\label{eq:def_gjs_simp}
&N  \KL{Q_1}{P} + n\KL{Q_2}{P}\nonumber\\
&= N\mathbb{E}_{Q_1}\left[\log \frac{Q_1}{P}\right] + n\mathbb{E}_{Q_2}\left[\log \frac{Q_2}{P}\right] \\
&=N\mathbb{E}_{Q_1}\left[\log \frac{Q_1}{\frac{\frac{N}{n}Q_1+Q_2}{1+\frac{N}{n}}}\right] + N\mathbb{E}_{Q_1}\left[\frac{\frac{\frac{N}{n}Q_1+Q_2}{1+\frac{N}{n}}}{P}\right] \nonumber \\
 & + n\mathbb{E}_{Q_2}\left[\log \frac{Q_2}{\frac{\frac{N}{n}Q_1+Q_2}{1+\frac{N}{n}}}\right]  + n\mathbb{E}_{Q_2}\left[\frac{\frac{\frac{N}{n}Q_1+Q_2}{1+\frac{N}{n}}}{P}\right]\\
&= n \GJS{Q_1}{Q_2}{\frac{N}{n}} + \left(N+n\right) \KL{\frac{\frac{N}{n}Q_1+Q_2}{1+\frac{N}{n}}}{P}. \label{eqn:DDG}
\end{align}
Then, from \cite{csiszar1998method}, we have
\begin{align}
&P^{ n+N }\left(v^N,w^n\right) = \nonumber \\ &  \exp\big(-N\left( \KL{Q_1}{P}+\ent{Q_1}\right)-n\left(\KL{Q_2}{P}+\ent{Q_2}\right)\big)\label{eq:def_gjs_1}  \\
&= \exp\bigg(-n \GJS{Q_1}{Q_2}{\frac{N}{n}}-N \ent{Q_1}- n \ent{Q_2} - \nonumber \\
&\quad\left(N+n\right) \KL{\frac{\frac{N}{n}Q_1+Q_2}{1+\frac{N}{n}}}{P}\bigg) \label{eq:def_gjs_2} \\
&\leq \exp\left(-n \GJS{Q_1}{Q_2}{\frac{N}{n}}-N \ent{Q_1} - n \ent{Q_2}\right) \label{eqn:drop_KL}
\end{align}
In (\ref{eq:def_gjs_2}), we plug (\ref{eq:def_gjs_simp}) into (\ref{eq:def_gjs_1}). Finally, the last step follows due to non-negativity of KL divergence. 
The upper bound can be achieved by setting $P = \frac{\frac{N}{n}Q_1 + Q_2}{1+\frac{N}{n}}$. Therefore, we can conclude that the optimal solution is $P^{\star} = \frac{\frac{N}{n}Q_1 + Q_2}{1+\frac{N}{n}}$.
\end{proof}
\end{lemma}
Intuitively, if one wishes to find a probability measure that maximizes the joint probability of observing two sequences with \text{different length} and \textit{different types}, then GJS naturally arises as a lower bound for the exponent of the desired probability.

In the next lemma, an interesting connection between the GJS divergence and mutual information is established.
\begin{lemma}\label{lem:mutual_gjs}
Let $X$ and $Y$ are two independent random variables drawn according to probability distributions $P$ and $Q$ respectively over the same alphabet. Also, $W$ is a Bernoulli random variable independent of $X$ and $Y$ with probabilities $\frac{\alpha}{1+\alpha}$ and $\frac{1}{1+\alpha}$ for $\alpha \in \mathbb{R}_{+}$, respectively. Let us define the random variable $Z$ as 
$
Z \triangleq \begin{cases} X & \text{if $W=0$} \\ 
 Y & \text{if $W=1$}  \end{cases}.
$  
Then, we have
\begin{equation} \label{eq:lem_mutual_info}
(1+\alpha) I\left(Z;W\right)= \GJS{P}{Q}{\alpha}
\end{equation}
\begin{proof}
We have
\begin{equation}
\begin{aligned}
&I\left(Z;W\right) = \ent{Z} - \ent{Z\big| W} \\
&= \ent{Z}-\ent{Z\big| W=0}\frac{\alpha}{1+\alpha} - \ent{Z\big|W=1}\frac{1}{1+\alpha}  \\
&=-\sum_{x\in \mathcal{X}} \left(\frac{\alpha}{1+\alpha}P\left(x\right) + \frac{1}{1+\alpha}Q\left(x\right) \right)  \\
& \hspace{1.5cm} \times \log\left(\frac{\alpha}{1+\alpha}P\left(x\right) + \frac{1}{1+\alpha}Q\left(x\right)\right) \\
&+\frac{\alpha}{1+\alpha}\sum_{x \in \mathcal{X}} P\left(x\right)\log P\left(x\right) + \frac{1}{1+\alpha} \sum_{x \in \mathcal{X}} Q\left(x\right) \log Q\left(x\right)\\
&=\frac{1}{1+\alpha} \GJS{P}{Q}{\alpha}
\end{aligned}
\end{equation}
which gives us the desired result in (\ref{eq:lem_mutual_info}). 
\end{proof}
\end{lemma}
Lemma \ref{lem:g} provides several properties of the GJS divergence. %
\begin{lemma}\label{lem:g}
For any pair of distributions $P$ and $Q$ in the interior of $\mathcal{P}\left(\mathcal{X}\right)$ and $\alpha \in \mathbb{R}_{+}$, we have the following facts.
\begin{enumerate}
\item $\GJS{P}{Q}{\alpha}$ is a concave function in $\alpha$ for fixed $P$ and $Q$. Moreover, $\lim_{\alpha \rightarrow \infty}\GJS{P}{Q}{\alpha} = \KL{Q}{P}  $.
\item For a fixed $\alpha \in\bbR_+$, $\GJS{P}{Q}{\alpha}$ is a jointly convex function in ($P$,$Q$).
\item For fixed $P$ and $Q$, the necessary and the sufficient condition for the equation $\GJS{P}{Q}{\alpha}=\lambda \alpha$ to have a non-zero solution is $\lambda < \KL{P}{Q}$. Also, the solution is unique.
\end{enumerate}  
\begin{proof} 
$\bullet$  \underline{ Proof of Part (1)}:
To begin with we start by deriving the first and the second derivatives of $\GJS{P}{Q}{\alpha}$. We have
\begin{align}
\frac{\partial  \GJS{P}{Q}{\alpha}}{\partial\alpha} &= \KL{P}{\frac{\alpha P+Q}{1+\alpha}} \label{eq:first_der}\\
\frac{\partial^2 \GJS{P}{Q}{\alpha}}{\partial\alpha^2} &= \frac{1}{1+\alpha}\sum_{x\in \mathcal{X}} P\left(x\right) \frac{Q\left(x\right) -  P\left(x\right)}{\alpha P\left(x\right) + Q\left(x\right) } \label{eq:sec_der}
\end{align}
We can manipulate the second derivative as
\begin{align}
&\frac{1}{1+\alpha}\sum_{x\in \mathcal{X}} P\left(x\right)\frac{Q\left(x\right)-P\left(x\right)}{\alpha P\left(x\right) + Q \left(x\right)} \nonumber \\
&= \frac{1}{1+\alpha}\sum_{x\in \mathcal{X}} P \left(x\right) \left(1-\frac{\left(1+\alpha\right) P \left(x\right)}{\alpha P \left(x\right)+Q \left(x\right)}\right) \nonumber\\
&=\frac{1}{1+\alpha} \left(1-\mathbb{E}_{P}\left[\frac{\left(1+\alpha\right)P \left(X\right)}{\alpha P\left(X\right)+Q \left(X\right)}\right]\right) \nonumber\\
&\leq \frac{1}{1+\alpha} \left(1-\mathbb{E}_{P}\left[\frac{\alpha P \left(X\right)+Q\left(X\right)}{\left(1+\alpha\right)P\left(X\right)}\right]^{-1}\right) \label{eq:jensen}\\
&=0 \nonumber ,
\end{align}
where (\ref{eq:jensen}) is obtained obtained by applying Jensen's inequality to the convex function ${1}/{x}$ for $x >0$. Thus, we conclude that the second derivative in (\ref{eq:sec_der}) is negative.  
Moreover, letting $\alpha$ tend to infinity in (\ref{eq:gjs_def}), we obtain the claim stated in the first part of Lemma \ref{lem:g}.\\
$\bullet$ \underline{ Proof of Part (2)}: Consider two pairs of distributions ($P_1$, $Q_1$) and ($P_2$, $Q_2$). For $ 0 \leq \theta \leq 1$, define $P_{\theta} = \theta P_1 + \left(1-\theta\right) P_2$ and $Q_{\theta} = \theta Q_1 + \left(1-\theta\right) Q_2$. Then, consider
\begin{align*}
&\GJS{P_{\theta}}{Q_{\theta}}{\alpha}= \nonumber \\
& \alpha \KL{\theta \mathrm{P}_1 + \left(1-\theta\right) \mathrm{P}_2 } {\theta \frac{\alpha \mathrm{P}_1 + \mathrm{Q}_1}{1+\alpha}+\left(1-\theta\right)\frac{\alpha \mathrm{P}_2 + \mathrm{Q}_2}{1+\alpha}} \nonumber\\
&+\KL{\theta \mathrm{Q}_1 + \left(1-\theta\right) \mathrm{Q}_2}{ \theta \frac{\alpha \mathrm{P}_1 + \mathrm{Q}_1}{1+\alpha}+\left(1-\theta\right)\frac{\alpha \mathrm{P}_2 + \mathrm{Q}_2}{1+\alpha}}\\
&\leq \theta \GJS{P_1}{Q_1}{\alpha} + \left(1-\theta\right) \GJS{P_2}{Q_2}{\alpha} 
\end{align*}
where the last step follows due to the convexity of KL divergence \cite[Thm. 2.7.2]{cover}.\\
$\bullet$ \underline{ Proof of Part (3)}: Let $f\left(\alpha\right) \triangleq \GJS{P}{Q}{\alpha}-\lambda\alpha$.
First, note that $f\left(0\right)=0$  and $f\left(\alpha\right)$ is a concave function. It is straightforward to see that $f\left(\alpha\right)$ has at most one non-zero root. Assume that there exists $\alpha^{\star} > 0$ such that $f\left(\alpha^{\star}\right)=0$. By the mean value theorem, we know there exist a $\tilde{\alpha} \in \left(0, \alpha^{\star}\right)$ such that $f'\left(\tilde{\alpha}\right)=0$. Knowing this fact and considering the strict concavity of $f\left(\alpha\right)$, we must have $f'\left(0\right)> 0$ which using (\ref{eq:first_der}) we have $\lambda < \KL{P}{Q}$. For the other direction, assume that $\lambda< \KL{P}{Q}$. Then, there exist an $\epsilon > 0$ such that for $0<\alpha< \epsilon$ we have $f\left(\alpha\right) > 0$. Then considering the fact that $\lim_{\alpha \to \infty } f\left(\alpha\right)=-\infty$, $f$ must have a root in the interval $[\epsilon, \infty)$.   %
\end{proof}
\end{lemma}
Lemma \ref{lem:error_same} states an upper bound on  the probability that the GJS divergence of two sequences (of different lengths in general) drawn from the same probability distribution exceeds $\gamma N$.
\begin{lemma}\label{lem:error_same}
Assume $X^N$ and $Y^n$ are two sequences drawn from the \textit{same} probability distribution $P\in \mathcal{P}\left(\mathcal{X}\right)$. Then, we have
\begin{equation}
\begin{aligned}
&\mathbb{P}\big( n\GJS{\ED{X^N}}{\ED{Y^n}}{\frac{N}{n}} \geq \gamma N\big)\leq  \exp\left(-\gamma  N\right) (n+N+1)^{|\mathcal{X}|}.
\end{aligned}
\end{equation}
\begin{proof}
Define $\mathcal{R}=\big\{ (Q_1,Q_2)\in \mathcal{P}_N\left(\mathcal{X}\right)\times \mathcal{P}_n\left(\mathcal{X}\right)\,|\, n \GJS{Q_1}{Q_2}{\frac{N}{n}} \geq \gamma N\big\}$. From the method of types, we can write
\begin{align}
&\mathbb{P}\left(n\GJS{\ED{X^N}}{\ED{Y^n}}{\frac{N}{n}} \geq \gamma N\right)\nonumber\\
&\leq \sum_{\left(Q_1,Q_2\right)\in \mathcal{R}}\exp\left(-N \KL{Q_1}{P}\right)\exp\left(-n \KL{Q_2}{P}\right) \label{eq:question_r1}\\
&= \sum_{\left(Q_1,Q_2\right)\in \mathcal{R}} \exp\left(-n \GJS{Q_1}{Q_2}{\frac{N}{n}}\right) \nonumber \\ 
& \hspace{1cm} \times \exp\left(-\left(N+n\right) \KL{\frac{\frac{N}{n}Q_1+Q_2}{1+\frac{N}{n}}}{P}\right) \label{eq:same_hamon}\\
&\leq \exp\left(-N\gamma\right)\sum_{\left(Q_1,Q_2\right)\in \mathcal{R}} \exp\big(-\left(N+n\right) \KL{\frac{\frac{N}{n}Q_1+Q_2}{1+\frac{N}{n}}}{P}\big) \nonumber \\
&\leq \exp\left(-N\gamma\right) \left(N+n+1\right)^{|\mathcal{X}|}  \sum_{\left(Q_1,Q_2\right)\in \mathcal{R}} \mathbb{P}\big(  \begin{bmatrix}
X^N\\ Y^n
\end{bmatrix} \in\Gamma^{n+N}_{\frac{NQ_1/n+Q_2}{1+N/n}}\big) \label{eq:method_of_type_new}\\
&\leq  \exp\left(-N\gamma\right) \left(N+n+1\right)^{|\mathcal{X}|},
\end{align}
where the first step is due to the independence of the two sequences and an application of Sanov's theorem.  Equation (\ref{eq:same_hamon}) is obtained by using (\ref{eqn:DDG}) and in (\ref{eq:method_of_type_new}), we   used \cite[Theorem~11.1.4]{cover} concerning  the probability of a type class.  
\end{proof}
\end{lemma}
\begin{lemma}\label{lem:conv_root}
Consider two probability distributions $P\in \mathcal{P}\left(\mathcal{X}\right)$ and $Q\in \mathcal{P}\left(\mathcal{X}\right)$. Fix $\gamma>0$ such that $C(P,Q)> \gamma$ and $C \in \mathbb{R}$ as an arbitrary constant. Denote $\alpha^{\star}$ as the solution of $\GJS{Q}{P}{\alpha^{\star}}=\gamma \alpha^{\star}$. Let $X^N$ denote a sequence consisting of $N$ i.i.d. samples drawn according to $Q$. Also,  $\alpha^{\star}_N$ denote the solution of $\GJS{\ED{X^N}}{P}{\alpha^{\star}_N}=\gamma \alpha^{\star}_N$ if exists, otherwise set $\alpha^{\star}_N = C$.  Then, as $N$ tends to infinity  $\alpha^{\star}_N$ converges in probability to  $\alpha^{\star}$.
\begin{proof}
Consider $\epsilon>0$. Define $\mathcal{S}_N=\{\ED{X^N} \in \mathcal{T}_{N} \big| \KL{\ED{X^N}}{P}\geq \gamma \}$. From Lemma \ref{lem:g}~Part 3, we know that under the event $\mathcal{S}_N$, there exists a solution for $\GJS{\ED{X^N}}{P}{\alpha^{\star}_N}=\gamma \alpha^{\star}_N$. Then, we can write
\begin{align}
&\mathbb{P}(|\alpha_N^{\star}-\alpha^\star|\geq\epsilon) = \mathbb{P}(\{|\alpha_N^{\star}-\alpha^\star|\geq\epsilon \} \cap \{\ED{X^N} \in \mathcal{S}_N\} ) \nonumber \\
&+ \mathbb{P}( \{|\alpha_N^{\star}-\alpha^\star|\geq\epsilon \} \cap \{\ED{X^N} \notin \mathcal{S}_N\} ) \nonumber  \\
&\leq  \mathbb{P}(\{|\alpha_N^{\star}-\alpha^\star|\geq\epsilon \} \cap \{\ED{X^N} \in \mathcal{S}_N\})  + \mathbb{P}(\ED{X^N} \notin \mathcal{S}_N). \label{eq:prob-root}
\end{align}
Let $f\left(V,\alpha\right): \mathcal{P}(X) \times \mathbb{R} \to \mathbb{R}$ defined as $f\left(V,\alpha\right)\triangleq \GJS{V}{P}{\alpha}-\gamma\alpha$. First of all note that since $\KL{P}{Q}\geq C(P,Q) > \gamma$, $\alpha^\star$ exists (see Lemma \ref{lem:g} Part 3). Also, note that on the event $\mathcal{S}_N$ we know there exists a solution $\alpha^{\star}_N$ such that it satisfies $f(\ED{X^N},\alpha^\star_N)=0$. From the implicit function theorem \cite{spivak2018calculus}, since $\frac{\partial f}{\partial \alpha}|_{\alpha=\alpha^\star}\neq 0 $ and $f$ is  a continuously differentiable function, there exists a unique continuously differentiable function $g : U \to \mathbb{R}$ and a open set $U$ which contains $Q$ such that $g(V)=\alpha$ and $f(V,g(V))=0$ for every $V\in U$. Therefore, we can find a sufficiently small $\delta > 0$ such that $\|\ED{X^N}-Q\|\leq \delta$  then $|\alpha_N^\star-\alpha^\star|\leq \epsilon$. Note that we need to choose $\delta$ so that $\{\|\ED{X^N}-Q\|\leq \delta\} \subseteq U$.
Hence the first term of  (\ref{eq:prob-root}) can be written as
\begin{equation}
\begin{aligned}
&\mathbb{P}(\{|\alpha_N^{\star}-\alpha^\star|\geq\epsilon \} \cap \{\ED{X^N} \in \mathcal{S}_N\})\\
& \leq \mathbb{P}(\{\|\ED{X^N}-Q \| \geq \delta \} \cap \{\ED{X^N} \in \mathcal{S}_N\}).
\end{aligned}
\end{equation}
Due to the fact that $\ED{X^N}$ converges in probability to $Q$ as $N \to \infty$, we conclude that the first term converges to zero as $N \to \infty$.
Using the Sanov's theorem, the second term in (\ref{eq:prob-root}) can be written as
\begin{equation}
\begin{aligned}
 \mathbb{P}(\ED{X^N} \notin \mathcal{S}_N) & \dotleq \exp(-N \min_{V: \KL{V}{P}\leq \gamma)} \KL{V}{Q}) \\
 &\dotleq \exp(-N\gamma)
 \end{aligned}
 \end{equation}
 The reason behind the last line is 
$
\left\lbrace \gamma\in\bbR \,\Big|\, \min_{V \in \mathcal{P}\left(\mathcal{X}\right):  \KL{V}{P} \leq \gamma} \KL{V}{Q} \geq \gamma \right\rbrace = [0, C\left(P,Q\right)].
$
Therefore we showed that both term in (\ref{eq:prob-root}) converges to zero as $N \to \infty$, as was to be shown.
\end{proof}
\end{lemma}

\begin{lemma} \label{lem:linear_opt}
Consider the optimization problem
\begin{equation}\label{eq:opt_linear}
\begin{aligned}
& \underset{ (\epsilon_1,\hdots,\epsilon_m )\in\bbR^m}{\text{min}}
& & \sum_{i=1}^{m} w_i\epsilon_i \\
& \text{subject to}
& & \sum_{i=1}^{m} |\epsilon_i| \leq \delta,\\
&&& \sum_{i=1}^{m} \epsilon_i = 0.
\end{aligned}
\end{equation}
Here, $  w_1,\hdots,w_m$ and $\delta > 0$ are constants. Then, the optimal value of the optimization problem in (\ref{eq:opt_linear}) is 
\begin{equation}
\frac{\delta}{2}\left(\min_{j\in \left[m\right]} w_{j} -\max_{j\in \left[m\right]} w_{j }\right).
\end{equation}
\end{lemma}
\begin{proof}
Note that in this proof, the optimal value of variables are denoted using an asterisk in the superscript. We assume without loss of generality  that $w_i$'s are all distinct. Otherwise, assume $w_{k_1}$ and $w_{k_2}$ are equal. Assume that we add another constraint to the optimization problem in (\ref{eq:opt_linear}) to get
\begin{equation}\label{eq:opt_linear__}
\begin{aligned}
& \underset{ (\epsilon_1,\hdots,\epsilon_m)\in\bbR^m }{\text{min}}
& & \sum_{i=1}^{m} w_i\epsilon_i \\
& \text{subject to}
& & \sum_{i=1}^{m} |\epsilon_i| \leq \delta,\\
&&& \sum_{i=1}^{m} \epsilon_i = 0\\
&&& \epsilon_{k_1}=\epsilon_{k_2}.
\end{aligned}
\end{equation}
We claim that the optimal value of the optimization problem in (\ref{eq:opt_linear__}) and (\ref{eq:opt_linear}), denoted by $\mathsf{OPT_1}$ and $\mathsf{OPT_2}$, are equal. The reason is as follows. First of all since the feasible set of (\ref{eq:opt_linear__}) is a subset of (\ref{eq:opt_linear}) we conclude that $\mathsf{OPT_1} \leq \mathsf{OPT_2}$. For the other direction, assume $\{\epsilon_{i,1}^{\star}\}_{i\in [m]}$ are the optimal value of (\ref{eq:opt_linear}). Then consider setting $\epsilon_i = \epsilon_{i,1}^{\star}$ for $i\notin\{k_1,k_2\}$ and $\epsilon_{k_1}=\epsilon_{k_2}=\frac{1}{2}{\epsilon_{k_1,1}^{\star}+\epsilon_{k_2,1}^{\star}} $. Since $\big| \epsilon_{k_1,1}^{\star}+\epsilon_{k_1,2}^{\star} \big|\leq \big| \epsilon_{k_1,1}^{\star}\big|+\big|\epsilon_{k_1,2}^{\star}\big|$, these values give us a feasible point for (\ref{eq:opt_linear__}). Thus, we have $\mathsf{OPT_1} \geq \mathsf{OPT_2}$. This result shows that given that $w_{k_1}$ and $w_{k_2}$ are equal we can write $w_{k_1}\epsilon_{k_1}+w_{k_2}\epsilon_{k_2}=2w_{k_1}\epsilon_{k_{1,2}}$ and instead of optimizing over $\epsilon_{k_1}$ and $\epsilon_{k_2}$ we can only optimize over $\epsilon_{k_{1,2}}$. So, in the sequel, we safely assume all $w_{i}$'s are distinct.
We introduce new variables $\epsilon^{+}_i\geq 0$ and $\epsilon^{-}_i \geq 0$ for $i \in \left[m\right]$. letting $\epsilon_i = \epsilon^{+}_i - \epsilon^{-}_i$, we rewrite (\ref{eq:opt_linear}) as
\begin{equation}\label{eq:opt_linear_without_abs}
\begin{aligned}
& \underset{ \epsilon_1,\hdots,\epsilon_m }{\text{min}}
& & \sum_{i=1}^{m} w_i \left(\epsilon^{+}_i - \epsilon^{-}_i\right) \\
& \text{subject to}
& & \sum_{i=1}^{m} \left(\epsilon^{+}_i + \epsilon^{-}_i\right) \leq \delta,\\
&&& \sum_{i=1}^{m} \epsilon^{+}_i = \sum_{i=1}^{m} \epsilon^{-}_i\\
&&& \epsilon^{+}_i,\epsilon^{-}_i\geq 0 \quad \forall i \in \left[m\right].
\end{aligned}
\end{equation}
Note that by $|\epsilon_i|=\epsilon_i^{+}+\epsilon_i^{-}$ and $\epsilon_i=\epsilon_i^{+}-\epsilon_i^{-}$ we implicitly impose the condition $\epsilon_i^{+}\epsilon_i^{-}=0$ without loss of optimality \cite{lin_opt}. The optimization problem in (\ref{eq:opt_linear_without_abs}) is a linear program, and the optimal solution can be found by considering the Karush-Kuhn-Tucker (KKT) conditions \cite[Chapter 5]{boyd}.  Then, we can write the Lagrangian function of (\ref{eq:opt_linear}) as
\begin{equation}
\begin{aligned}
&\mathcal{L}\left( \left\lbrace \epsilon^{+}_i\right\rbrace_{i=1}^{m},  \left\lbrace \epsilon^{-}_i\right\rbrace_{i=1}^{m}, \theta, \nu, \left\lbrace \lambda^{+}_i\right\rbrace_{i=1}^{m},\left\lbrace \lambda^{-}_i \right\rbrace_{i=1}^{m}\right)  \\ 
&=\sum_{i=1}^{m} w_i \left(\epsilon^{+}_i - \epsilon^{-}_i\right) + \nu\left(\sum_{i=1}^{m} \left(\epsilon^{+}_i + \epsilon^{-}_i\right)  - \delta \right) \nonumber \\
&- \sum_{i=1}^{m}\lambda_i^{+}\epsilon^{+}_i - \sum_{i=1}^{m}\lambda_i^{-}\epsilon^{-}_i + \theta \sum_{i=1}^{m}\left(\epsilon^{+}_i-\epsilon^{-}_i\right),
\end{aligned}
\end{equation}
where $\nu \geq 0$, $\lambda_i^{+} \geq 0$, $\lambda_i^{-} \geq 0$, and $\theta$ are dual variables. Taking the derivatives of Lagrangian with respect to $\epsilon^{+}_i$ and $\epsilon^{-}_i$ and setting them to zero, we get
\begin{align}
w_i + \nu^{\star} -\left(\lambda_i^{+}\right)^{\star} +\theta^{\star} & = 0, \label{eq:lag_+} \quad\mbox{and}\\
-w_i + \nu^{\star} -\left(\lambda_i^{-}\right)^{\star} -\theta^{\star} &= 0 \label{eq:lag_-},
\end{align}
respectively. Note that given that $w_1,\hdots,w_m$ are not all-zero, it can be verified that there exists at least two indices $i,j \in \left[m\right]$ such that we have $\epsilon^{\star}_i > 0$ and $\epsilon^{\star}_j < 0$. In this way, for $\epsilon^{\star}_i$ and $\epsilon^{\star}_j$, (\ref{eq:lag_+}) and (\ref{eq:lag_-}) can be written as
\begin{align}
w_i &= - \nu^{\star} -\theta^{\star},\label{eq:lag_+_2}\\
w_j &= \nu^{\star} -\theta^{\star}\label{eq:lag_-_2} .
\end{align}
In (\ref{eq:lag_+_2}) and (\ref{eq:lag_-_2}), we have used   complementary slackness, i.e., ${\lambda_i^{+}}^{\star} {\epsilon^{+}_i}^{\star} =0$ and ${\lambda_j^{-}}^{\star} {\epsilon^{-}_j}^{\star} =0$. 
Here we claim that there are exactly two indices for which $|\epsilon_i^{\star}|>0$. The reason is that as seen in (\ref{eq:lag_+_2}) and (\ref{eq:lag_-_2}), $-\nu^{\star}-\theta^{\star}$ and $\nu^{\star}-\theta^{\star}$ can only take two values. So, given that $w_i$'s are all distinct  (\ref{eq:lag_+_2}) and (\ref{eq:lag_-_2}) can only be satisfied by exactly two indices.  Therefore, searching among all $ {m \choose 2}$ combinations of choosing two  out of $m$ indices, it is straightforward to see that the optimal solution of (\ref{eq:opt_linear}) is given by
\begin{align}
\epsilon_i^{\star} = \begin{cases}
\frac{\delta}{2} ~ &\text{if} ~ i=\argmin_k w_k, \\
-\frac{ \delta}{2} ~ &\text{if} ~ i=\argmin_k w_k,\\
0 ~ &\text{else}.
\end{cases}
\end{align} 
Thus, the claim stated in the lemma follows.
\end{proof}

\subsection{Proof of the results for the binary case}
\label{sub:proof_main_b}
In this section, we provide the steps toward proving Theorem \ref{thm:main_b}. Specifically,  this section consists of three parts. First, we present our results on the expected value of the stopping time. Then, the error probability analysis is provided. Finally, we conclude with the derivation of the error exponent.

For brevity, we present the result for the case that the true hypothesis is $H_2$, i.e., the underlying probability measure is $\probm_2$. The extension of the results
here to the case that $H_1$ is the true hypothesis can be readily done by replacing $\probm_1$
by $\probm_2$ and vice versa. The following lemma will be used in the the next theorem. 
\begin{lemma}\label{lem:alpha_p_n}
Assume $f(x)=\GJS{\ED{X_1^N}}{P_2}{x}-\gamma x=0$ has a solution in $x$ and let $\theta_N^{\star}$ be the solution. Consider 
\begin{equation}\label{eq:opt_lower_bound_main}
\begin{aligned}
U_{+}\left(\theta\right) ~ \triangleq ~
 & \underset{V \in \mathcal{P}\left(\mathcal{X}\right)}{\text{max}}
& & \GJS{\ED{X_1^N}}{V}{\theta} \\
& \text{subject to}
& &  \KL{V}{P_2} \leq \frac{1}{\sqrt{N}}.
\end{aligned}
\end{equation}
and 
\begin{equation}\label{eq:opt_upper_bound_main}
\begin{aligned}
U_{-}\left(\theta\right)~ \triangleq ~ & \underset{V\in \mathcal{P}\left(\mathcal{X}\right)}{\text{min}}
& & \GJS{\ED{X_1^N}}{V}{\theta} \\
& \text{subject to}
& & \KL{V}{P_2}\leq \frac{1}{\sqrt{N}}.
\end{aligned}
\end{equation}
Then, for sufficiently large $N$, we can construct $\theta_N^{+}$ and $\theta_N^{-}$ which have the following properties.
\begin{enumerate}
\item $\theta_N^{+} > \theta_N^{\star}$
\item $\theta_N^{+} - \theta_N^{\star}=O\left(\frac{\log N}{N^{1/4}}\right)$
\item $\theta_N^{-} < \theta_N^{\star}$
\item $\theta_N^{\star} - \theta_N^{-}=O\left(\frac{1}{N^{1/4}}\right)$
\end{enumerate}
Furthermore, $\theta_N^{+}$ and $\theta_N^{-}$  satisfy
\begin{align}
U_{+}\left(\theta_N^{+}\right)&< \gamma \theta_N^{+} ,\quad\mbox{and}\\
U_{-}\left(\theta_N^{-}\right)&> \gamma \theta_N^{-}.
\end{align}
\begin{proof}
 The proof consists of explicit constructions of $\theta_N^{+}$ and $\theta_N^{-}$. The proof is tedious and deferred to
 Appendix \ref{app:theta_+_-}.
\end{proof}
\end{lemma}
Before presenting the main properties of the stopping time, we present the following lemma which provides an almost sure lower bound on the stopping time.
\begin{lemma} \label{lem:lower_val_stop}
The stopping time defined in (\ref{eq:def_stoptime}) is greater than $\left(\frac{\gamma}{2\log 2}\right)^2 N$ almost surely.
\begin{proof}
Define
\begin{equation}\label{eq:def_rv_lower}
Z_n \triangleq \begin{cases} Z_{n,1} & \text{if $W_n=1$} \\ 
 Z_{n,2} & \text{if $W_n=2$}  \end{cases},
\end{equation}
where $Z_{n,1}$ and $Z_{n,2}$ are distributed according to $\ED{X_1^N}$ and $\ED{Y^n}$, respectively. Also, following the same notation as in Lemma~\ref{lem:mutual_gjs}, let $W_n$ be a Bernoulli random variable  $\mathrm{Bern}(\frac{N}{n+N})$. From the results of Lemma \ref{lem:mutual_gjs}, we can write
\begin{align}
&n  \GJS{\ED{X_1^N}}{\ED{Y^n}}{\frac{N}{n}}=  n\left(1+\frac{N}{n}\right)I\left(Z_n; W_n\right) \label{eq:estefade_az_mutual}\\
&\leq n\left(1+\frac{N}{n}\right) \mathrm{H}_{b}\left(\frac{N}{N+n}\right) \label{eq:estefade_az_fact}\\
&\leq \left(n+N\right)( 2\log 2) \sqrt{\frac{Nn}{\left(n+N\right)^2}} \label{eq:approx_binary}\\
&\leq (2 \log 2) \sqrt{nN}\nonumber.
\end{align}
Here,   (\ref{eq:estefade_az_fact}) follows due to the fact that $I\left(Z_n; W_n\right)\leq \mathrm{H}_{b}\left(W_n\right)=\mathrm{H}_{b}\left(\frac{N}{n+N}\right)$ where $\mathrm{H}_{b}(p)$ is the binary entropy function defined as $\mathrm{H}_{b}(p)=-p\log p -\left(1-p\right)\log \left(1-p\right)$. Finally, in (\ref{eq:approx_binary}) we use  $\mathrm{H}_{b}\left(p\right)\leq (2\log 2)\sqrt{p\left(1-p\right)}$. Therefore, considering the stopping time in (\ref{eq:def_stoptime}), we can conclude that 
\begin{equation}\label{eq:lower_Bound}
T_{\text{seq}} \geq \left(\frac{\gamma}{2\log 2}\right)^2 N.
\end{equation}
By replacing $\ED{X_1^N}$ with $\ED{X_2^N}$ in the definition of the random variable in (\ref{eq:def_rv_lower}), we get the same lower bound as in (\ref{eq:lower_Bound}) so the lower bound is agnostic to the true hypothesis.
\end{proof}
\end{lemma}
Before presenting the next results, we will define a notation. Let $X$ and $T$ are two random variables, and $B$ is a $\sigma(X)$-measurable set. We define $\mathbb{E}[T|X, X\in B]$ as a $\sigma(X)$-measurable function as  $\mathbb{E}[T|X, X\in B](x) \triangleq \mathbb{E}[T|X](x) 1_{B}(x)$, where $1_B$ denotes the indicator function takes value of 1 on $B$.
\begin{lemma} \label{lem:stopping_time_2}
Define the set 
\begin{equation}
\mathcal{S}_1\triangleq \left\{\ED{X_1^N} \in \mathcal{T}_{N} \,\Big| \,\KL{\ED{X_1^N}}{P_2}\geq \gamma \right\} .
\end{equation}
 The stopping time in (\ref{eq:def_stoptime}) has the following properties:
\begin{enumerate}
\item $
\mathbb{E}_2\big[T_{\text{seq}} \,\big|\,\ED{X_1^N} , \ED{X_1^N}\in \mathcal{S}_1\big] \geq  \frac{N}{\theta^{+}_N}  \left(1- o(1)\right)$ 
\item $\mathbb{E}_2\big[T_{\text{seq}} \,\big|\, \ED{X_1^N}  ,\ED{X_1^N}\in \mathcal{S}_1\big] \leq  \frac{N}{\theta^{-}_N}  \left(1+ o(1)\right)$ 
\item $\mathbb{E}_2\big[T_{\text{seq}}\big]=\frac{N}{\theta^{\star}}\left(1+o\left(1\right)\right)$
\end{enumerate}
where  $\theta^{+}_N$ and $\theta^{-}_N$ are defined in Lemma \ref{lem:alpha_p_n}. Also, $\theta^{\star}$ is the solution of
\begin{equation}
\GJS{P_1}{P_2}{\theta^{\star}}=\gamma \theta^{\star}.
\end{equation} 
\begin{proof}
First of all, note that  for all $\ED{X_1^N}\in \mathcal{S}_1$ (the set $\calS_1$ was defined in the statement of Lemma \ref{lem:stopping_time_2}), we can assert that there exists a solution to the equation  $\GJS{\ED{X_1^N}}{P_2}{\theta_N^{\star}}=\gamma \theta_N^{\star}$ (See Part 3 of Lemma \ref{lem:g}).\\
\label{app:proof_avg_stop}
\underline{\textbf{Proof of Part 1:}}
We obtain
\begin{align}
&\mathbb{E}_2\left[T_{\text{seq}}\big| \ED{X_1^N}, \ED{X_1^N}\in \mathcal{S}_1\right]  =\sum_{k\geq1}\mathrm{P}_2\left(T_{\text{seq}}\geq k \,\Big|\,  \ED{X_1^N}, \ED{X_1^N}\in \mathcal{S}_1\right)  \nonumber\\
&\geq \sum_{k=1}^{ \frac{N}{\theta_N^{+}}} \mathrm{P}_2\left(T_{\text{seq}}\geq k \,\Big|\,\ED{X_1^N}, \ED{X_1^N}\in \mathcal{S}_1\right)  \nonumber \\
& \geq  \frac{N}{\theta_N^{+}} \big(1-\mathrm{P}_2\big( 1 \leq T_{\text{seq}} \leq  \frac{N}{\theta_N^{+}} \big| \ED{X_1^N},\ED{X_1^N}\in \mathcal{S}_1\big)\big)  \nonumber\\
&=  \frac{N}{\theta_N^{+}} \big(1-\mathrm{P}_2\big(  \big(\frac{\gamma}{2\log 2}\big)^2 N  \leq T_{\text{seq}} \leq  \frac{N}{\theta_N^{+}} \big|\ED{X_1^N}, \ED{X_1^N}\in \mathcal{S}_1\big) \label{eq:bounded_time},
\end{align}
where in (\ref{eq:bounded_time}) we have used Lemma \ref{lem:lower_val_stop}.  Next, we show that the probability term in \eqref{eq:bounded_time} is $o(1)$. %
To do so, we obtain \eqref{eq:error_gheir_same} and \eqref{eq:error_bade} on the top of the next page.
\newcounter{MYtempeqncnt}
\begin{figure*}[!t]
\normalsize
\setcounter{MYtempeqncnt}{\value{equation}}
\setcounter{equation}{73}
\begin{small}
\begin{align}
&\mathrm{P}_2\big( (\frac{\gamma}{2\log2})^2 N   \leq T_{\text{seq}} \leq  \frac{N}{\theta_N^{+}} \big| \ED{X_1^N},\ED{X_1^N}\in \mathcal{S}_1 \big) \leq \mathrm{P}_2\big(\bigcup\limits_{k=(\frac{\gamma}{2\log2})^2 N}^{\frac{N}{\theta_N^{+}}}\big\{ k \GJS{\ED{X_2^N}}{\ED{Y^k}}{\frac{N}{k}} \geq N\gamma \big\} \big| \ED{X_1^N}, \ED{X_1^N}\in \mathcal{S}_1\big) \nonumber\\
& +\mathrm{P}_2\big( \bigcup\limits_{k=\left(\frac{\gamma}{2\log2}\right)^2 N}^{\frac{N}{\theta_N^{+}}} \left\lbrace k \GJS{\ED{X_1^N}}{\ED{Y^k}}{\frac{N}{k}} \geq N\gamma \right\rbrace \bigg| \ED{X_1^N}, \ED{X_1^N}\in \mathcal{S}_1\big) \label{eq:error_gheir_same}\\
&\leq   \frac{N}{\theta_N^{+}} \exp\left(-\gamma N\right) \bigg( \frac{N}{\theta_N^{+}}  +N+1\bigg)^{|\mathcal{X}|} + \sum_{k= \left(\frac{\gamma}{2\log2}\right)^2 N } ^{  \frac{N}{\theta_N^{+}} }\mathrm{P}_2\left(k \GJS{\ED{X_1^N}}{\ED{Y^k}}{\frac{N}{k}} \geq \gamma N  \bigg| \ED{X_1^N}, \ED{X_1^N}\in \mathcal{S}_1\right)\label{eq:error_bade}.
\end{align}
\end{small}
\setcounter{equation}{\value{MYtempeqncnt}}
\hrulefill
\vspace*{-4pt}
\end{figure*}
\addtocounter{equation}{2}
(\ref{eq:error_gheir_same}) is due to the definition of the stopping time in (\ref{eq:def_stoptime}) and the union bound. 
To obtain the first term in (\ref{eq:error_bade}), the result of Lemma \ref{lem:error_same} is used. Then, we provide an upper bound for the second term in (\ref{eq:error_bade}) as follows. Let $f: \mathbb{N}\times \mathcal{P}(\mathcal{X})\to \mathbb{R}$ be 
$$
f(k,\ED{X_1^N}) = \min_{V\in \mathcal{P}\left(\mathcal{X}\right): k \GJS{\ED{X_1^N}}{V}{\frac{N}{k}} \geq \gamma N } \KL{V}{P_2}.
$$
Then, consider
\begin{align}
 &\sum_{k= \left(\frac{\gamma}{2\log2}\right)^2 N } ^{  \frac{N}{\theta_N^{+}} }\mathrm{P}_2\big(k  \GJS{\ED{X_1^N}}{\ED{Y^k}}{\frac{N}{k}} \geq N\gamma\big| \ED{X_1^N}, \ED{X_1^N}\in \mathcal{S}_1\big) \label{eq:stop_time_lower_step_1}\\
  &\leq  \sum_{k= \left(\frac{\gamma}{2\log2}\right)^2 N } ^{  \frac{N}{\theta_N^{+}} } \left(k+1\right)^{|\mathcal{X}|} \exp\big(-k  f(k,\ED{X_1^N}) \big) \label{eq:question_r1_cond1}\\
& \leq  \sum_{k= \left(\frac{\gamma}{2\log2}\right)^2 N } ^{  \frac{N}{\theta_N^{+}} } \big(\frac{N}{\theta_N^{+}} +1\big)^{|\mathcal{X}|} \exp\big(- (\frac{\gamma}{2\log2})^2 N  f(\frac{N}{\theta_N^{+}},\ED{X_1^N})\big) \label{eq:last_step_lower_bound}\\
&\leq \left(\frac{N}{\theta_N^{+}}+1\right)^{|\mathcal{X}|+1} \exp\big(- (\frac{\gamma}{2\log2})^2 \sqrt{N}\big), \label{eq:stop_time_lower_step_final}
\end{align}
where in (\ref{eq:last_step_lower_bound}) we have used the fact that the function $k\mathrm{GJS}\left(\mathrm{P},\mathrm{Q},\frac{N}{k}\right)$ is increasing with $k$. Therefore, as we increase $k$, the value of the optimization problem in (\ref{eq:last_step_lower_bound}) decreases. Finally, the last step comes from the property of $\theta_N^{+}$ in Lemma \ref{lem:alpha_p_n} which argues that \begin{equation}
\begin{aligned}
\gamma \theta_{N}^{+}~ \geq  ~ & \underset{V\in \mathcal{P}\left(\mathcal{X}\right)}{\text{max}}
& & \GJS{\ED{X_1^N}}{V}{\theta_{N}^{+}} \\
& \text{subject to}
& & \KL{V}{P_2}\leq \frac{1}{\sqrt{N}}.
\end{aligned}
\end{equation}
This completes the proof of Part 1 of Lemma~\ref{lem:stopping_time_2}.\\

\underline{\textbf{Proof of Part 2:}}
We begin with bounding the tail probability of the stopping time. We can write
\begin{align}
&\mathrm{P}_2\left(T_{\text{seq}} > \frac{N}{\theta_N^{-}}\bigg| \ED{X_1^N}, \ED{X_1^N}\in \mathcal{S}_1\right)\\
&=\sum_{k\geq\frac{N}{\theta_N^{-}}} \mathrm{P}_2\left(T_{\text{seq}} = k+1\bigg| \ED{X_1^N}, \ED{X_1^N}\in \mathcal{S}_1\right)\label{eq:upper_exp_1}\\
&\leq \sum_{k\geq\frac{N}{\theta_N^{-}}} \mathrm{P}_2\left(k \GJS{\ED{X_1^N}}{\ED{Y^k}}{\frac{N}{k}} \leq \gamma N \bigg| \ED{X_1^N}, \ED{X_1^N}\in \mathcal{S}_1\right)\\
&\leq \sum_{k\geq\frac{N}{\theta_N^{-}}}  \left(k+1\right)^{|\mathcal{X}|}   \exp\big(-k  \min_{V\in  \mathcal{P}\left(\mathcal{X}\right): k \GJS{\ED{X_1^N}}{V}{\frac{N}{k}}\leq \gamma N  }  \KL{V}{P_2}\big)\\
& \leq \sum_{k\geq\frac{N}{\theta^{-}_N}}  \exp\left(|\mathcal{X}|\log\left(k+1\right)\right) \exp\left(-\frac{k}{\sqrt{N}}\right)\label{eq:upper_exp_on}\\
&\leq \exp\left(-\frac{\sqrt{N}}{\theta^{-}_N}\left(1+o(1)\right)\right).\label{eq:upper_exp_on_last_next}
\end{align}
Here, (\ref{eq:upper_exp_on}) is obtained using the results of Lemma \ref{lem:alpha_p_n} and the fact that $k\mathrm{GJS}\left(\mathrm{P},\mathrm{Q},\frac{N}{k}\right)$ is an increasing function in $k$. Then the last step follows from some manipulations. Finally, from (\ref{eq:upper_exp_1})-(\ref{eq:upper_exp_on_last_next}), we deduce that 
\begin{align}
&\mathbb{E}_2\left[T_{\text{seq}}\bigg|\ED{X_1^N}, \ED{X_1^N}\in \mathcal{S}_1\right] \nonumber\\
 &\leq \frac{N}{\theta^{-}_N} \mathrm{P}_2\left(T_{\text{seq}}\leq \frac{N}{\theta^{-}_N} \bigg| \ED{X_1^N},\ED{X_1^N}\in \mathcal{S}_1\right) + \nonumber \\
 & \sum_{k\geq\frac{N}{\theta^{-}_N} } \left(k+1\right)\mathrm{P}_2\left(T_{\text{seq}}=k+1\bigg|\ED{X_1^N}, \ED{X_1^N}\in \mathcal{S}_1\right)  \\
&\leq \frac{N}{\theta^{-}_N}\left(1+o(1)\right),\label{eq:exp_upper_bound}
\end{align}
which is the desired result.\\

\underline{\textbf{Proof of Part 3:}}
By the construction of $\theta_N^{+}$ and $\theta_N^{-}$  in the proof of Lemma \ref{lem:alpha_p_n}, when $N$ diverges to infinity, it can be seen that $\theta_N^{+}-\theta_N^{\star}=O\left(\frac{\log N}{N^{\frac{1}{4}}}\right)$ and $\theta_N^{\star}-\theta_N^{-}=O\left(\frac{1}{N^{\frac{1}{4}}}\right)$. Also, from Lemma \ref{lem:conv_root}, we know that $\theta_N^{\star}$ converges in probability to $\theta^{\star}_\gamma$. Considering the definition of $T_{\text{seq}}$ in (\ref{eq:def_stoptime}), we can write
\begin{align}
&\mathbb{E}_2\left[T_{\text{seq}}\right]=\mathbb{E}_2 [\mathbb{E}_2[T_{\text{seq}}\,|\,\ED{X_1^N} ]  \mathbbm{1}\{\ED{X_1^N}\in \mathcal{S}_1\}] \nonumber \\
&+ \mathbb{E}_2[\mathbb{E}_2[T_{\text{seq}}\,|\,  \ED{X_1^N}] \mathbbm{1}\{\ED{X_1^N}\notin \mathcal{S}_1\}] \\
& = \frac{N}{\theta^{\star}_\gamma} \left(1+o(1)\right) + N^2\mathrm{P}_2\left(\ED{X_1^N}\in \mathcal{S}_1\right)\label{eq:cov-prob-to-expect}\\
&= \frac{N}{\theta^{\star}_\gamma} \left(1+o(1)\right) + o(1)
\end{align}
Here, for the first term of (\ref{eq:cov-prob-to-expect}) we have used \cite[Thm. 2.3.4]{durrett2019probability} to leverage the convergence in probability for $\theta^\star_N \to \theta^\star$ into convergence in expectation. 
Note for the final step we have used (\ref{eq:chernoof_binary}) and Sanov's theorem to write 
\begin{equation}
\mathrm{P}_2\left(\ED{X_1^N} \not\in \mathcal{S}_1\right) \dotleq  \exp\left(-N\gamma\right)
\end{equation}
and $T_{\text{seq}}$ is, almost surely, at most $N^2$, which is subexponential.
\end{proof}
\end{lemma}
\begin{corollary}
\label{corr:stop_binary_complete}
When the true hypothesis is $H_1$, we have
\begin{equation}
\mathbb{E}_1\left[T_{\text{seq}}\right]=\frac{N}{\beta^{\star}_\gamma}\left(1+o(1)\right),
\end{equation}
where, $\beta^{\star}_\gamma$ is the solution of
\begin{equation}
\GJS{P_2}{P_1}{\beta^{\star}_\gamma}=\gamma \beta^{\star}_\gamma.
\end{equation} 
\end{corollary}
Therefore, considering the results in Part 3 of Lemma \ref{lem:stopping_time_2} and Corollary \ref{corr:stop_binary_complete}, the claim in Theorem \ref{thm:main_b} regarding the stopping time follows immediately.

The following lemma presents bounds on the error probability of the proposed test.
\begin{lemma} \label{lem:error_2}
Under the two different hypotheses, the error probabilities of $\Phi_{\text{seq}}$ satisfy
\begin{align} 
&\pe_1 (\Phi_\text{seq}(\gamma)) \dotleq \nonumber \\
&  \exp\big(-N\min\big\{ \gamma, \min_{V \in \mathcal{P}\left(\mathcal{X}\right):  \KL{V}{P_1} \leq \gamma + \varepsilon_N} \KL{V}{P_2} \big\}\big). \label{eqn:min1}\\
& \pe_2 (\Phi_\text{seq}(\gamma))
\dotleq \nonumber \\
& \exp\big(-N\min\big\{ \gamma, \min_{V \in \mathcal{P}\left(\mathcal{X}\right):  \KL{V}{P_2} \leq \gamma + \varepsilon^{\prime}_N} \KL{V}{P_1} \big\}\big).\label{eqn:min2}
\end{align}
where $\varepsilon_N$ and $\varepsilon^{\prime}_N$ are sequences that tend  to zero as $N \rightarrow \infty$.
\begin{proof}

To compute error probability, we define test $\Phi_{\text{trunc}}(\gamma)$ as a truncated version of $\Phi_{\text{seq}}(\gamma)$. Using $\Phi_{\text{trunc}}(\gamma)$, the decision maker follows the same decision rule as $\Phi_{\text{seq}}(\gamma)$ in the interval $\left[1,N^2\right]$. However, if the stopping time $T_{\text{seq}}$ has not occurred in the interval $\left[1,N^2\right]$, the decision maker declares error. It is easy to verify that the error probability of $\Phi_{\text{trunc}}(\gamma)$ is an upper bound for that of $\Phi_{\text{seq}}(\gamma)$. Hence, we can write
\begin{align}
& \pe_2(\Phi_{\text{seq}}(\gamma))\leq \pe_2 (\Phi_\text{trunc}(\gamma)) \\
&= \probm_2\left( \bigcup\limits_{k=1}^{N^2}\left\lbrace n\GJS{\ED{X_2^N}}{\ED{Y^k}}{\frac{N}{k}}\geq \gamma N\right\rbrace \right) \nonumber \\
&+ \probm_2\left( T_{\text{seq}}\geq N^2 \right)\label{eq:error_prob_binary_}
\end{align}
where the first and second term in (\ref{eq:error_prob_binary_}) correspond  to the events of ``wrong decision" and ``no decision"  respectively. From Part 3 of Lemma \ref{lem:g}, we know that in order to to have a $\theta^{\star}_N$ which satisfies  $\GJS{\ED{X_1^N}}{P_2}{\alpha}=\gamma\alpha$, we require the condition $\KL{\ED{X_1^N}}{P_2} \geq \gamma$. Also, from the results of Lemma \ref{lem:alpha_p_n} we know that $\theta_N^{-}$ can be constructed using $\theta^{\star}_N$ given that $\theta^{\star}_N$ exists. In fact, the map between $\theta_N^{\star}$ and $\theta_N^{-}$ is one-to-one. Let us define the following set
\begin{equation}
\begin{aligned}
\mathcal{A}_N\triangleq \bigg\{ &\ED{X_1^N}\in \mathcal{T}_N \big| \exists\, \theta^{\star}_N ~ \text{such that} \\ 
&\GJS{\ED{X_1^N}}{P_2}{\theta_N^{\star}}=\gamma \theta_N^{\star} ~ \text{and} ~ \theta_N^{-}\geq \frac{N}{N^2} \bigg\}
\end{aligned}
\end{equation}
Next, we argue that since  $\theta_N^{\star}$ is a continuous function of $\gamma$, $\mathcal{A}_N$ has another representation which is given by 
\begin{equation}
\mathcal{A}_N =\left\lbrace \ED{X_1^N}\in \mathcal{T}_N \,\Big|\, \KL{\ED{X_1^N}}{P_2} \geq \gamma + \varepsilon _N \right\rbrace
\end{equation}
where $\varepsilon_N \geq 0$ goes to zero as $N$ goes to infinity because as $N$ goes to infinity, $\theta_{N}^{\star}$ is greater than zero, and this condition can be satisfied by having $\KL{\ED{X_1^N}}{P_2} > \gamma$ (See Lemma \ref{lem:g}).  Then, we can write
\begin{align}
& \probm_2\left( T_{\text{seq}}\geq N^2 \Big|\ED{X_1^N}, \ED{X_1^N} \in \mathcal{A}_N \right) \nonumber \\
&\leq  \sum_{k\geq N^2 }  \exp\left(|\mathcal{X}|\log\left(k+1\right)\right) \exp\left(-\frac{k}{\sqrt{N}}\right) \label{eq:err_prob_line11}\\
& \leq  \exp\left(-\frac{N^2}{\sqrt{N}}\left(1+o\left(1\right)\right)\right)\\
& \leq \exp\left(-N^{\frac{3}{2}}\left(1+o\left(1\right)\right)\right) \label{eq:err_prob_lineff}
\end{align}
where in (\ref{eq:err_prob_line11}) we have used (\ref{eq:upper_exp_1})-(\ref{eq:upper_exp_on}) and the fact that $N^2\geq \frac{N}{\theta_N^{-}}$. Then, we obtain
\begin{align}
&\probm_2\left( T_{\text{seq}}\geq N^2\right)   \nonumber\\
&\leq  \mathbb{E}_2\left[\probm_2\left( T_{\text{seq}}\geq N^2 \Big| \ED{X_1^N}  \right) \mathbbm{1}\{\ED{X_1^N} \in \mathcal{A}_N\}\right]  \nonumber\\
&+ \probm_2\left(\KL{\ED{X_1^N}}{P_2}\leq \gamma +  \varepsilon_N\right) \\
& \leq \exp\left(-N^{\frac{3}{2}}\left(1+o\left(1\right)\right)\right) +  \probm_2\left(\KL{\ED{X_1^N}}{P_2}\leq \gamma +  \varepsilon_N \right)\label{eq:step_t_seq_bound} \\
&\leq \exp\left(-N^{\frac{3}{2}}\left(1+o\left(1\right)\right)\right) \nonumber \\
&+ \exp\left(- N\min_{V \in \mathcal{P}\left(\mathcal{X}\right):  \KL{V}{P_2} \leq \gamma + \varepsilon_N} \KL{V}{P_1}\right).
\end{align}
Here, in (\ref{eq:step_t_seq_bound}), we have used (\ref{eq:err_prob_lineff}). Also, the last step follows from Sanov's theorem. Therefore, we obtain
\begin{align}
&\pe_2(\Phi_{\text{seq}}(\gamma))  \nonumber \\
& \leq  N^2 \exp\left(- N \gamma\right) \left(N+N^2+1\right)^{|\mathcal{X}|} 
 +  \exp\left(-N^{\frac{3}{2}}\left(1+o\left(1\right)\right)\right)  \nonumber\\
 &\qquad + \exp\left(- N\min_{V \in \mathcal{P}\left(\mathcal{X}\right):  \KL{V}{P_2} \leq \gamma + \varepsilon_N} \KL{V}{P_1}\right)\label{eq:error_prob_last_step}\\
&\dotleq  \exp\left(-N\min\left\lbrace \gamma, \min_{V \in \mathcal{P}\left(\mathcal{X}\right):  \KL{V}{P_2} \leq \gamma + \varepsilon_N} \KL{V}{P_1} \right\rbrace\right),
\end{align}
where the first term in (\ref{eq:error_prob_last_step}) follows by Lemma \ref{lem:error_same}.
\end{proof}
\end{lemma}

Equipped  with the analysis of the stopping time and error probability, we conclude the proof of Theorem \ref{thm:main_b} by deriving the desired achievable error exponent. We write
\begin{align}
&\textsf{e}_2\left(\Phi_{\text{seq}}(\gamma)\right)=\liminf\limits_{N\to \infty}\frac{-\log \pe_2(\Phi_{\text{seq}}(\gamma)) }{\mathbb{E}_2\left[T_{\text{seq}}\right]}\\
&\geq  \theta^{\star} \liminf\limits_{N\to \infty} \ \min \big\{ \gamma, 
\min_{V \in \mathcal{P}\left(\mathcal{X}\right):  \KL{V}{P_2} \leq \gamma + \varepsilon_N} \KL{V}{P_1} \big\} \label{eq:ashish_com}  \\
&= \theta^{\star} \min \big\{ \gamma, 
\min_{V \in \mathcal{P}\left(\mathcal{X}\right):  \KL{V}{P_2} \leq \gamma} \KL{V}{P_1} \big\}\label{eq:vincent_com}\\
&= \theta^{\star} \gamma  \label{eq:condition}\\
&= \mathrm{GJS}\left(P_1,P_2,\theta^{\star}_\gamma\right). \label{eqn:last_step2}
\end{align}
Here, in (\ref{eq:ashish_com}), we have used Lemmas \ref{lem:stopping_time_2} and \ref{lem:error_2}. The equality in~(\ref{eq:vincent_com}) follows from the continuity of the optimal value of the optimization problem with respect to $\varepsilon_N$ \cite[Sec 5.6]{boyd}. In (\ref{eq:condition}), we   used the fact that 
\begin{equation}
\label{eq:chernoof_binary}
\big\{ \gamma\,\big|\, \min_{V \in \mathcal{P}\left(\mathcal{X}\right):  \KL{V}{P_2} \leq \gamma} \KL{V}{P_1} \geq \gamma \big\} = [0, C\left(P_1,P_2\right)].
\end{equation}
Finally, the last step in~\eqref{eqn:last_step2} is obtained due to the defintion of $\theta_\gamma^\star$ in (\ref{eq:fixed2}). Note that the extension of the results here to the type-I error exponent can be readily done which leads to the statement in Theorem \ref{thm:main_b}.

\subsection{Proof of Theorem \ref{thm:comp_gut}}
\label{subsec:proof_of_comp}
In this part, we denote $\Phi_{\text{GUT},1}$ to denote the test described in (\ref{eq:gutman_test_tweeked}). The subscript $1$ in $\Phi_{\text{GUT},1}$ represents the fact that the test uses the first training sequence. In this subsection, we prove Theorem \ref{thm:comp_gut} which states that the proposed test outperforms the Gutman's test in terms of Bayesian error exponent defined in (\ref{eq:bayesian_exponent}). We first begin with proving a property of the Gutman's test which will be used in the main proof.

For the test described in (\ref{eq:gutman_test_tweeked}), we have 
\begin{align}
 \mathsf{E}_1\left(\Phi_{\text{GUT},1}\right) &\triangleq \liminf_{N \to \infty} \frac{-\log \pe_1(\Phi_{\text{GUT,1}})}{N} \geq \lambda,\quad\mbox{and} \\
 \mathsf{E}_2\left(\Phi_{\text{GUT},1}\right)&\triangleq\liminf_{N \to \infty} \frac{-\log \pe_2(\Phi_{\text{GUT,1}})}{N}\geq F_1\left(\alpha,\lambda\right).
\end{align}
  A schematic of $\min\{  \mathsf{E}_1\left(\Phi_{\text{GUT},1}\right),  \mathsf{E}_2\left(\Phi_{\text{GUT},1}\right)\}$ versus $\lambda$ is depicted in Figure \ref{fig:gut_1}. Two important observations are in order.
\begin{itemize}
\item For $\lambda \geq \frac{1}{\alpha}{\GJS{P_1}{P_2}{\alpha}}$, we have $\min\{ \mathsf{E}_1\left(\Phi_{\text{GUT},1}\right),\mathsf{E}_2\left(\Phi_{\text{GUT},1}\right)\}=0$ as a consequence of (\ref{eq:opt_gut_N}) .
\item  $\lambda_1^{\star}$ in Fig. \ref{fig:gut_1} denotes the maximum achievable Bayesian error exponent as defined in (\ref{eq:bayes_1}). 
\end{itemize}
\begin{figure}[t]
\centering
\begin{tikzpicture}[scale=1.2]
    \draw [<->,thick] (0,1.5) node (yaxis) [above] {$\min\{  \mathsf{E}_1\left(\Phi_{\text{GUT},1}\right),  \mathsf{E}_2\left(\Phi_{\text{GUT},1}\right)\}$}
        |- (3.4,0) node (xaxis) [right] {$ \mathsf{E}_1\left(\Phi_{\text{GUT},1}\right)$};
    \draw[black, line width = 0.50mm] (0,0) coordinate (a_1) -- (1,1) coordinate (a_2);
    \draw[black, line width = 0.50mm]   plot[smooth,domain=1:3] (\x, {0.25*(\x-3)^2});
    \coordinate (c) at (1,1);
	\coordinate (d) at (3,0);    
    \node[circle,inner sep=1.2pt,fill=black,label=below:{$\frac{\mathrm{GJS}\left(P_1,P_2,\alpha\right)}{\alpha}$}] at (3,0) {};
    \draw[dashed] (yaxis |- c) node[left] {$\lambda_1^{\star}$}
        -| (xaxis -| c) node[below] {$\lambda_1^{\star}$};
\end{tikzpicture}
\caption{ The performance of the Gutman's test when the \textit{first} training sequence is used. }
\label{fig:gut_1}
\end{figure}
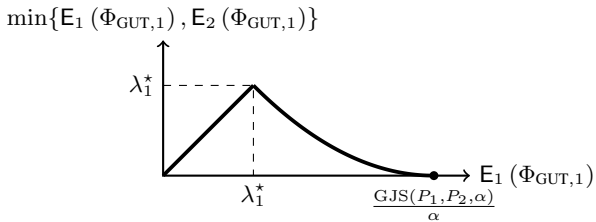
It is important to note that although two training sequences are
produced, only one of them  $X_1^N$  is used in (\ref{eq:gutman_test_tweeked}). One can suggest the following test which resembles the one in (\ref{eq:gutman_test_tweeked}) but uses the second training sequence as
\begin{equation} \label{eq:gutman_test_2}
\Phi_\text{GUT,2}=
\begin{cases} 
  H_2 & \text{if } \GJS{\ED{X_2^N}}{\ED{Y^n}}{\alpha} \leq \lambda \alpha, \\
   H_1       & \text{if } \GJS{\ED{X_2^N}}{\ED{Y^n}}{\alpha} \geq \lambda \alpha ,
\end{cases}.
\end{equation}
Note that the $\Phi_{\text{GUT,1}}$ and $\Phi_{\text{GUT,2}}$ depend on $\alpha$ and $\lambda$, but we do not want to show the dependence due to the notational convenience. The extension of the Gutman's main theorem to the test in (\ref{eq:gutman_test_2}) is given by the following lemma.
\begin{lemma} \label{lem:gut_v2}
Among all decision rules  $\Phi$
such that for all pairs of distribution $(P_1, P_2)\in \mathcal{P}\left(\mathcal{X}\right)^2$, 
\begin{equation}
\liminf_{N \to \infty} \frac{-\log \pe_2(\Phi_\text{GUT,2}(\lambda,\alpha))}{N}\geq \lambda,
\end{equation}
 the test $\Phi_\text{GUT,2}$ in (\ref{eq:gutman_test_2}) satisfies 
 \begin{equation}
 \liminf_{N \to \infty} \frac{-\log \pe_1(\Phi_{\text{GUT,2}})}{N}\geq \liminf_{N \to \infty} \frac{-\log \pe_1(\Phi)}{N}.
 \end{equation} Also, given that $\alpha=\frac{N}{n}$, we obtain 
\begin{align}
&\mathsf{E}_1\left(\Phi_{\text{GUT},2}\right)=\liminf_{N \to \infty} \frac{-\log \pe_1(\Phi_{\text{GUT,2}})}{N} \geq F_2\left(\alpha, \lambda\right),\label{eq:typeI_gut_2}\\
&\mathsf{E}_2\left(\Phi_{\text{GUT},2}\right)=\liminf_{N \to \infty} \frac{-\log \pe_2(\Phi_{\text{GUT,2}})}{N} \geq \lambda, 
\end{align}
where 
\begin{equation}\label{eq:opt_gut_2}
\begin{aligned}
F_2(\alpha, \lambda)\triangleq & \underset{\left(V_1, V_2\right)\in \mathcal{P}\left(\mathcal{X}\right)^2}{\text{min}}
& & \KL{V_1}{P_2} + \frac{1}{\alpha} \KL{V_2}{P_1} \\
& \text{subject to}
& & \frac{1}{\alpha}\mathrm{GJS}\left(V_1, V_2, \alpha \right)\leq \lambda.
\end{aligned}
\end{equation}
\end{lemma}

\begin{figure}[t]
\centering
\begin{tikzpicture}[scale=1.7]
    \draw [<->,thick] (0,1.5) node (yaxis) [above] {$\min\{ \mathsf{E}_1\left(\Phi_{\text{GUT},2}\right),\mathsf{E}_2\left(\Phi_{\text{GUT},2}\right)\}$}
        |- (2.4,0) node (xaxis) [right] {$\mathsf{E}_2\left(\Phi_{\text{GUT},2}\right)$};
    \draw[black, line width = 0.50mm] (0,0) coordinate (a_1) -- (1,1) coordinate (a_2);
    \draw[black, line width = 0.50mm]   plot[smooth,domain=1:2] (\x, {1*(\x-2)^2});
    \coordinate (c) at (1,1);
	\coordinate (d) at (2,0);    
    \node[circle,inner sep=1.2pt,fill=black,label=below:{$\frac{\mathrm{GJS}\left(P_2,P_1,\alpha\right)}{\alpha}$}] at (2,0) {};
    \draw[dashed] (yaxis |- c) node[left] {$\lambda_2^{\star}$}
        -| (xaxis -| c) node[below] {$\lambda_2^{\star}$};
\end{tikzpicture}
\caption{The performance of the Gutman's test when the \textit{second} training sequence is used. }
\label{fig:gut_2}
\end{figure}
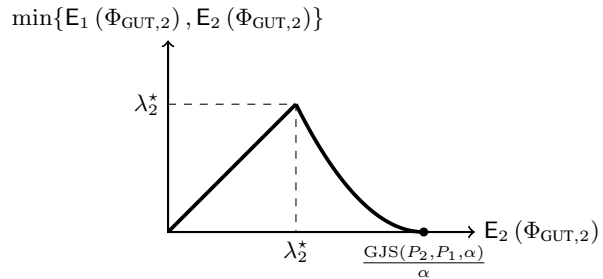
In Figure \ref{fig:gut_2}, we show a schematic plot of $\min\{ \mathsf{E}_1\left(\Phi_{\text{GUT},2}\right),\mathsf{E}_2\left(\Phi_{\text{GUT},2}\right)\}$ versus  $\mathsf{E}_2\left(\Phi_{\text{GUT},2}\right)$. Similar to Figure \ref{fig:gut_1}, we observe that
\begin{itemize}
\item For $\lambda \geq \frac{1}{\alpha}{\GJS{P_2}{P_1}{\alpha}}$, we have $\min\{ \mathsf{E}_1\left(\Phi_{\text{GUT},2}\right),\mathsf{E}_2\left(\Phi_{\text{GUT},2}\right)\}=0$.
\item  Also, $\lambda_2^{\star}$ in Fig. \ref{fig:gut_2} depicts the maximum achievable Bayesian error exponent of $\Phi_{\text{GUT},2}$.
\end{itemize} 
\begin{lemma} \label{lem:bayes_agn}
The maximum achievable Bayesian error exponents of $\Phi_{\text{GUT},1}$ and  $\Phi_{\text{GUT},2}$ are equal. 
\end{lemma}
Hence, the Bayesian error exponent of   Gutman's test is agnostic to which training sequence is being used. 
\begin{proof}
Here, we want to prove that $\lambda_1^{\star}=\lambda_2^{\star}$. The proof is by contradiction. Assume that $\lambda_1^{\star}<\lambda_2^{\star}$. Consider the tradeoff of type-I and type-II error exponents in Figure \ref{fig:gut_2}. Then, denote $\lambda^{+}$ as the solution to $F_2\left(\alpha,\lambda^{+}\right)=\lambda_1^{\star}$. Since $\lambda_1^{\star} < \lambda_2^{\star}$ and $F_2\left(\alpha,\lambda\right)$ is decreasing function in $\lambda$, it can be verified that $\lambda^{+}\in \left(\lambda_2^{\star},\frac{1}{\alpha}\GJS{P_2}{P_1}{\alpha}   \right)$. Therefore, we have $\lambda^{+} > \lambda_2 ^{\star}>\lambda_1^{\star}$. Here, we want to prove that $\lambda^{+}$ being greater than $\lambda_1^{\star}$  contradicts with optimality of Gutman's test described in Theorem \ref{thm:gut1}. Assume that $\lambda^+>\lambda_1^\ast$, we can argue that the test based on the second training sequence achieves the type-I error exponent equal to $\lambda_1^{\star}$ while its type-II error exponent is $\lambda^{+}> \lambda_1^{\star}$. This contradicts with the fact that among all tests that achieve the same type-I error exponent, Gutman's test has the largest type-II exponent. By the same argument, it can be shown that $\lambda_2^\ast<\lambda_1^\ast$ contradicting Lemma \ref{lem:gut_v2}. Thus, we have $\lambda_1^{\star}$ in Fig.~\ref{fig:gut_1} is equal to $\lambda_2^{\star}$ in Fig.~\ref{fig:gut_2}.
\end{proof}
Now, the result of Lemma \ref{lem:bayes_agn} allows us to prove Theorem \ref{thm:comp_gut}. Consider  the following two scenarios separately.
\begin{enumerate}
\item \underline{$\theta^{\star}_\gamma\leq \beta^{\star}_\gamma$}: Given that $\theta^{\star}\leq \beta^{\star}$, we have $\gamma=\frac{1}{\theta^{\star}}{\GJS{P_1}{P_2}{\theta^{\star}}}$ as shown in Theorem \ref{thm:main_b}, and $\gamma$ is the maximum achiavable exponent of $\Phi_\text{seq}(\gamma)$. Considering $\alpha=\theta^{\star}$ for Gutman's test, 
$$
\lambda_1^\star \stackrel{(a)}{<} \GJS{P_1}{P_2}{\alpha}/\alpha\stackrel{(b)}{=} \GJS{P_1}{P_2}{\theta^{\star}_\gamma }/\theta^{\star}_\gamma \stackrel{(c)}{=}\gamma.
$$
Here, $(a)$ is by Figure \ref{fig:gut_1}, $(b)$ follows since $\alpha=\min\{\theta^{\star}_\gamma, \beta^{\star}_{\gamma}\}=\theta^\star_\gamma$, and $(c)$ is due to Theorem \ref{thm:main_b}. 
\item \underline{$\theta^{\star}_\gamma > \beta^{\star}_\gamma$}: In this case, for the sequential test we have  $\gamma=\frac{1}{\beta^{\star}}{\GJS{P_2}{P_1}{\beta^{\star}}}$, and . 
$$
\lambda_2^\star \stackrel{(a)}{<} \GJS{P_2}{P_1}{\alpha}/\alpha\stackrel{(b)}{=} \GJS{P_2}{P_1}{\beta^{\star}_\gamma }/\beta^{\star}_\gamma \stackrel{(c)}{=}\gamma.
$$
Here, $(a)$ is by Figure \ref{fig:gut_2}, $(b)$ follows since $\alpha=\min\{\theta^{\star}_\gamma, \beta^{\star}_{\gamma}\}=\beta^\star_\gamma$, and $(c)$ is due to Theorem \ref{thm:main_b}. 
 Then, using the fact that $\lambda_2^{\star}=\lambda_1^{\star}=\textsf{e}^{\pi}_{\text{Bayesian}}\left(\Phi_{\text{GUT}}(\lambda^\star,\alpha)\right)$, the claim stated in Theorem \ref{thm:comp_gut} is proved. 
\end{enumerate}
\subsection{Proof of the results for multi-class classification problem}
\label{sec:proof_multi}
This section consists of three parts: stopping time analysis, derivation of the error probability, and finally characterizing the achievable error exponent.

Our main result on the expected value $T_{\text{seq}}^{(M)}$ is presented in the next lemma.
\begin{lemma}
\label{lem:stop_time_multi}
Denote $\theta^{\star}_{i\left(j\right),\gamma}$ as the solution of the equation
\begin{equation} \label{eq:fix_point_multi}
\GJS{P_j}{P_i}{\theta^{\star}_{i\left(j\right),\gamma}}=\gamma \theta^{\star}_{i\left(j\right),\gamma}, \quad  j \in \left[ M \right] , i\neq j.
\end{equation}
Then, the expected value of $T_{\text{seq}}^{(M)}$ satisfies
\begin{equation}
\mathbb{E}_i\left[T_{\text{seq}}^{(M)}\right]=\frac{N}{\min_{j\in \left[M\right],j\neq i}\{\theta^{\star}_{i(j),\gamma}\}}\left(1+o\left(1\right)\right),
\end{equation}
for all $i \in \{1,\hdots,M\}$.
\begin{proof}
Let us assume that the test sequence generated from $P_1$, i.e., belongs to Class $1$. The extension to other cases is straightforward. Define the set
\begin{equation}
\begin{aligned}
    \mathcal{S}_{1}^{(M)} \triangleq  &\bigg\{\left( \ED{X_2^N}, \hdots,\ED{X_M^N} \right) \bigg| \left( \ED{X_2^N}, \hdots,\ED{X_M^N} \right) \in \\ 
    &\prod_{i=2}^{M} \left\{\ED{X_i^N} \in \mathcal{T}_{N} \,\Big| \,\KL{\ED{X_i^N}}{P_1}\geq \gamma \right\} \bigg\} 
\end{aligned}
\end{equation}
Conditioned on $\mathcal{S}_{1}^{(M)}$, we can find  $\theta_{N,1(j)}^{\star}$ such that $\theta_{N,1(j)}^{\star}$ satisfies
\begin{equation}
\GJS{\ED{X_j^N}}{P_1}{\theta_{N,1(j)}^{\star}}=\gamma \theta_{N,1(j)}^{\star} \quad j \in \{2,\hdots,M\}.
\end{equation}
Also, define
\begin{align} 
\theta^{\star}_{N,1} &\triangleq \min_{j\in \{2,\hdots,M\}}\{ \theta_{N,1(j)}^{\star} \},  \quad\mbox{and}\label{eq:def_stop_fixed_train}\\
j^{\star}&\triangleq \argmin_{j\in \{2,\hdots,M\}}\{ \theta_{N,1(j)}^{\star} \}. \label{eq:def_stop_fixed_train_indx}
\end{align}
In addition, we substitute $P_2$ with $P_1$ and $\ED{X_1^N}$ with $\ED{X^N_{j^{\star}}}$ in Lemma \ref{lem:alpha_p_n} to obtain  $\theta_{N,1(j^{\star})}^{+}$ and $\theta_{N,1(j^{\star})}^{-}$ following the same procedure as described in Lemma \ref{lem:alpha_p_n}.  
We start with providing a lower bound on the expected value of the stopping time. We can write
\begin{align}
&\mathbb{E}_1\left[T_{\text{seq}}^{(M)}\big| \{\ED{X_j^N}\}_{2\leq j\leq M},  \mathcal{S}_{1}^{(M)} \right]\nonumber\\
&=\sum_{k=1}^{\infty}\probm_1\left(T_{\text{seq}}^{(M)}\geq k\big|  \{\ED{X_j^N}\}_{2\leq j\leq M},\mathcal{S}_{1}^{(M)} \right)\\
&\geq \sum_{k=1}^{\frac{N}{\theta_{N,1(j^{\star})}^{+}}}\probm_1\left(T_{\text{seq}}^{(M)}\geq k\big| \{\ED{X_j^N}\}_{2\leq j\leq M},  \mathcal{S}_{1}^{(M)}\right)\\
& \geq \frac{N}{\theta_{N,1(j^{\star})}^{+}}   \big(1-\probm_1\big(1\leq T_{\text{seq}}^{(M)}\leq \frac{N}{\theta_{N,1(j^{\star})}^{+}}\big| \{\ED{X_j^N}\}_{2\leq j\leq M},  \mathcal{S}_{1}^{(M)} \big)\big)\\
&= \frac{N}{\theta_{N,1(j^{\star})}^{+}} \big(1- \nonumber \\
&  \probm_1\big( \big(\frac{\gamma}{2\log 2}\big)^2 N \leq T_{\text{seq}}^{(M)}\leq \frac{N}{\theta_{N,1(j^{\star})}^{+}}\big|  \{\ED{X_j^N}\}_{2\leq j\leq M}, \mathcal{S}_{1}^{(M)}\big)\big) \label{eq:lower_stop_m_1}				
\end{align}
where in the last step we have used Lemma \ref{lem:lower_val_stop}. Moreover, define
\begin{align}
\tau_i^{(M)}&= \inf \left\lbrace n\geq 1 ~:~ n \GJS{\ED{X_i^N}}{\ED{Y^n}}{\frac{N}{n}} \geq \gamma N \right\rbrace
\end{align} 
as the time that empirical GJS divergence between the test sequences and the $i$-th training sequence exceeds the threshold. 
 Then, we upper bound the probability term in \eqref{eq:lower_stop_m_1} as shown on the top of next page in \eqref{eq:first-step-auxiallry}-\eqref{eq:stop_m_final_lower}.
\begin{figure*}[!t]
\normalsize
\setcounter{MYtempeqncnt}{\value{equation}}
\setcounter{equation}{133}
\begin{small}

\begin{align}
&\probm_1\big( (\frac{\gamma}{2\log 2})^2 N \leq T_{\text{seq}}^{(M)} \leq \frac{N}{\theta_{N,1(j^{\star})}^{+}} \bigg| \{\ED{X_j^N}\}_{2\leq j\leq M},  \mathcal{S}_{1}^{(M)} \big) \nonumber\\
&\leq \probm_1\big( (\frac{\gamma}{2\log 2})^2 N \leq \max\{\tau^{(M)}_2,\hdots,\tau^{(M)}_M\} \leq \frac{N}{\theta_{N,1(j^{\star})}^{+}},\bigcap\limits_{n=\left(\frac{\lambda}{2\log 2}\right)^2 N}^{ \frac{N}{\theta_{N,1(j^{\star})}^{+}}}\left\lbrace n \GJS{\ED{X_1^N}}{\ED{Y^n}}{\frac{N}{n}} \leq \gamma N\right)\big| \{\ED{X_j^N}\}_{2\leq j\leq M},  \mathcal{S}_{1}^{(M)} \big)\nonumber\\
&\qquad+ \probm_1\big(\bigcup\limits_{n=\left(\frac{\gamma}{2\log 2}\right)^2 N}^{ \frac{N}{\theta_{N,1(j^{\star})}^{+}}} \left\lbrace n\GJS{\ED{X_1^N}}{\ED{Y^n}}{\frac{N}{n}}\geq \gamma N\right)\big| \{\ED{X_j^N}\}_{2\leq j\leq M},  \mathcal{S}_{1}^{(M)}  \big) \label{eq:first-step-auxiallry} \\
&\leq \probm_1\left( \left(\frac{\gamma}{2\log 2}\right)^2 N \leq \tau^{(M)}_{j^{\star}} \leq \frac{N}{\theta_{N,1(j^{\star})}^{+}}\bigg| \{\ED{X_j^N}\}_{2\leq j\leq M},  \mathcal{S}_{1}^{(M)}\right) + \probm_1\big(\bigcup\limits_{n=\left(\frac{\gamma}{2\log 2}\right)^2 N}^{ \frac{N}{\theta_{N,1(j^{\star})}^{+}}} \big\{ n\GJS{\ED{X_1^N}}{\ED{Y^n}}{\frac{N}{n}}\geq \gamma N \big\} \big)\\
&\leq \left(\frac{N}{\theta_{N,1(j^{\star})}^{+}}+1\right)^{|\mathcal{X}|+1} \exp\left(- \left(\frac{\gamma}{2\log2}\right)^2 \sqrt{N}\right) +  \frac{N}{\theta_{N,1(j^{\star})}^{+}} \exp\left(-\gamma N\right) \left( \frac{N}{\theta_{N,1(j^{\star})}^{+}}  +N+1\right)^{|\mathcal{X}|} \label{eq:stop_m_before} \\ 
&= o\left(1\right). \label{eq:stop_m_final_lower}
\end{align}
\end{small}
\setcounter{equation}{\value{MYtempeqncnt}}
\hrulefill
\vspace*{-4pt}
\end{figure*}
\addtocounter{equation}{4} 
Here, the first term on the LHS of (\ref{eq:stop_m_before}) is obtained by the same  reasons  as those for   (\ref{eq:stop_time_lower_step_1})-(\ref{eq:stop_time_lower_step_final}). Also, the second term on the LHS of (\ref{eq:stop_m_before}) follows from Lemma \ref{lem:error_same}. Thus, we conclude from (\ref{eq:lower_stop_m_1}) and (\ref{eq:stop_m_final_lower}) that
\begin{equation}
\mathbb{E}_1\left[T_{\text{seq}}^{(M)}\big|  \{\ED{X_j^N}\}_{2\leq j\leq M}, \mathcal{S}_{1}^{(M)}  \right] \geq \frac{N}{\theta_{N,1(j^{\star})}^{+}}\left(1-o\left(1\right)\right).
\end{equation} 
In the top of next page in \eqref{eq:stop-upper-bound-start}-\eqref{eq:stop_m_upper_bound_last}, an upper bound on the tail probability of $T_{\text{seq}}^{(M)}$ is derived which leads to an upper bound on the expected value of $T_{\text{seq}}^{(M)}$.
\begin{figure*}[!t]
\normalsize
\setcounter{MYtempeqncnt}{\value{equation}}
\setcounter{equation}{138}
\begin{small}

\begin{align}
&\probm_1\left(T_{\text{seq}}^{(M)}> \frac{N}{\theta_{N,1(j^{\star})}^{-}}\bigg|   \{\ED{X_j^N}\}_{2\leq j\leq M}, \mathcal{S}_{1}^{(M)} \right) =\sum_{k= \frac{N}{\theta_{N,1(j^{\star})}^{-}}}^{\infty} \probm_1\left(T_{\text{seq}}^{(M)}=k+1\bigg|  \{\ED{X_j^N}\}_{2\leq j\leq M},  \mathcal{S}_{1}^{(M)} \right)\label{eq:stop-upper-bound-start}\\
&\leq  \sum_{k= \frac{N}{\theta_{N,1(j^{\star})}^{-}}}^{\infty} \sum\limits_{ \left(i_1,i_2\right) \in \left[M\right]^2,i_1\neq i_2}\probm_1\left( \tau_{i_1}^{(M)}>k,\tau_{i_2}^{(M)}>k \bigg|  \{\ED{X_j^N}\}_{2\leq j\leq M}, \mathcal{S}_{1}^{(M)}  \right) \label{eq:stop_m_upper_bound_} \\
&\leq \sum\limits_{ \left(i_1,i_2\right) \in \left[M\right]^2,i_1\neq i_2,i_1=1} \sum_{k= \frac{N}{\theta_{N,1(j^{\star})}^{-}}}^{\infty}   \probm_1\left(k \GJS{\ED{X_{i_2}^N}}{\ED{Y^k}}{\frac{N}{k}}\leq \gamma  N\bigg|  \{\ED{X_j^N}\}_{2\leq j\leq M},\mathcal{S}_{1}^{(M)}  \right) \nonumber \\
&\qquad+  \sum\limits_{ \left(i_1,i_2\right) \in \left[M\right]^2,i_1\neq i_2,i_1 \neq 1} \sum_{k= \frac{N}{\theta_{N,1(j^{\star})}^{-}}}^{\infty}  \probm_1\left(k \GJS{\ED{X_{i_1}^N}}{\ED{Y^k}}{\frac{N}{k}} \leq \gamma N\bigg|  \{\ED{X_j^N}\}_{2\leq j\leq M},  \mathcal{S}_{1}^{(M)} \right)\\
&\leq   \frac{M\left(M-1\right)}{2} \exp\left(-\frac{\sqrt{N}}{\theta_{N,1(j^{\star})}^{-}}\left(1+o(1)\right)\right) \label{eq:stop_m_upper_bound_last}.
\end{align}
\end{small}
\setcounter{equation}{\value{MYtempeqncnt}}
\hrulefill
\vspace*{-4pt}
\end{figure*}
\addtocounter{equation}{4} 
Here, (\ref{eq:stop_m_upper_bound_}) is obtained using the fact that the event $ \{ T_{\text{seq}}^{(M)}=k+1 \}$ has the same probability as the event that there exists at least two indices $\left(i,j\right) \in \left[M\right]^2$ such that $\tau_i^{(M)}>k$ and $\tau_j^{(M)}>k$.  Then, (\ref{eq:stop_m_upper_bound_last}) follows from the definition of $j^{\star}$ in (\ref{eq:def_stop_fixed_train_indx}) which attains the minima. Then, following the same line of reasoning as (\ref{eq:exp_upper_bound}) we obtain
\begin{equation}
\mathbb{E}_1\left[T^{(M)}_{\text{seq}}\,\Big|\,  \{\ED{X_j^N}\}_{2\leq j\leq M},  \mathcal{S}_{1}^{(M)}  \right]\leq \frac{N}{\theta_{N,1(j^{\star})}^{-}}\left(1+o\left(1\right)\right)
\end{equation} 
Finally, note that as it was proved in Lemma \ref{lem:conv_root}, $\theta^{\star}_{N,1(j)}$ converges in probability to $\theta_{1(j)}^{\star}$  for $j \in \{2,\hdots,M\}$ as $N$ goes to infinity. Also, since $\min$ function is continuous in its argument, the continuous mapping theorem 
\cite{billingsley2008probability} implies that $\theta^{\star}_{N,1}$ converges in probability to $\min_{j\in \left[M\right],j\neq 1}\{\theta^{\star}_{1(j),\gamma}\}$. To conclude the proof, we write
\begin{align}
&\mathbb{E}_1\left[T_{\text{seq}}^{(M)}\right]=\nonumber\\
&\mathbb{E}_1 [\mathbb{E}_1[T_{\text{seq}}^{(M)}\,|\, \{\ED{X_j^N}\}_{2\leq j\leq M} ]  \mathbbm{1}\{\{\ED{X_j^N}\}_{2\leq j\leq M}\in \mathcal{S}_{1}^{(M)} \}] + \nonumber\\
& \mathbb{E}_1[\mathbb{E}_1[T_{\text{seq}}\,|\,  \{\ED{X_j^N}\}_{2\leq j\leq M}] \mathbbm{1}\{\{\ED{X_j^N}\}_{2\leq j\leq M}\notin \mathcal{S}_{1}^{(M)} \}] \\
& \leq  \frac{N}{\min_{j\in \left[M\right],j\neq 1}\{\theta^{\star}_{1(j),\gamma}\}} \left(1+o(1)\right) + N^2\mathrm{P}_1\left(\ED{X_1^N}\in \mathcal{S}_{1}^{(M)} \right) \label{eq:cov-prob-to-expect-multi}\\
&\leq  \frac{N}{\min_{j\in \left[M\right],j\neq 1}\{\theta^{\star}_{1(j),\gamma}\}} \left(1+o(1)\right) + o(1)
\end{align}
Note in the second term of the final step we have used (\ref{eq:chernoof_multi}) to write 
\begin{equation}
1-\probm_1\left(\{\ED{X_j^N}\}_{2\leq j\leq M} \in \mathcal{S}^{(M)}_1\right) \dotleq  \exp\left(-N\gamma\right)
\end{equation}
and $T_{\text{seq}}^{M}$ is, almost surely, at most $N^2$, which is subexponential.
\end{proof}
\end{lemma}

\begin{lemma}
\label{lem:error_multi}
The error probability of the test $\Phi_{\text{seq}}^{(M)}(\gamma)$ is given by
\begin{equation}
\begin{aligned}
&\pe_i(\Phi^{(M)}_{\text{seq}}(\gamma)) \dotleq \\
&\exp\big(-N\min\big\{ \gamma, \min_{j \in [M],j\neq i}  \min_{V:  \KL{V}{P_j} \leq \gamma + \varepsilon_{j,N}} \KL{V}{P_i} \big\}\big),
\end{aligned}
\end{equation}
where $\varepsilon_{j,N}\geq 0 $ is a sequence for each $j\in \{1,...,M\}$ converging to zero as $N$ tends to infinity for all $i \in \left[M\right]$.
\begin{proof}
We define a test $\Phi_{\text{trunc}}^{(M)}$ to be a truncated version of $\Phi_{\text{seq}}^{(M)}$ in an exactly similar way as we defined $\Phi_{\text{trunc}}^{(M)}$ in the proof of Lemma \ref{lem:error_2}.  Then, we can write
\begin{align}
 &\pe_1(\Phi_{\text{seq}}^{(M)})\leq \pe_1(\Phi_{\text{trunc}}^{(M)})\\
&\leq \probm_1\left( \bigcup\limits_{n=1}^{N^2}\left\lbrace n\GJS{\ED{X_1^N}}{\ED{Y^n}}{\frac{N}{n}}\geq \gamma N\right\rbrace \right) \nonumber \\
&\hspace{1.2cm}+ \probm_1\left( T^{(M)}_{\text{seq}}\geq N^2 \right)\label{eq:error_multi_22}.
\end{align}
Note that the first and the second term in (\ref{eq:error_multi_22}) correspond to the event ``wrong decision" and the ``no decision".   Following the same line of reasoning as in the proof of Lemma \ref{lem:error_2}, we consider the event $ \bigcap\limits_{i=2}^{M} \big\{\KL{\ED{X_i^N}}{P_1}> \gamma + \varepsilon_{i,N} \big\}$ where $\varepsilon_{i,N} \geq 0$ is a sequence goes to zero as $N$ goes to infinity. Conditioned on this event we can conclude that there exists  $\theta_{N,1\left(i\right)}^{\star}$  which satisfies equation $\GJS{\ED{X_i^N}}{P_1}{\theta_{N,1\left(i\right)}^{\star}}=\gamma \theta_{N,1\left(i\right)}^{\star}$ for $i \in \{2,\hdots,M\}$ by Part 3 of Lemma \ref{lem:g} . Define $\theta_{N,1\left(i^{\star}\right)}^{-}$ following the method described in Lemma \ref{lem:alpha_p_n}. Introducing $\varepsilon_{i,N}$ let us have $N^2\geq {N}/{\theta_{N,1(i)}^{-}}$ for $i \in \{2,\hdots,M\}$. Also, let $\theta^{\star}_{N,1\left(i^{\star}\right)} \triangleq \min_{i\in \{2,\hdots,M\}}\{ \theta_{N,1(i)}^{\star} \}$.  Then, we can write
\begin{align}
&\probm_1\big( T^{(M)}_{\text{seq}}\geq N^2 \big|  \{\ED{X_j^N}\}_{2\leq j\leq M},\bigcap_{i=2}^{M} \big\{ \KL{\ED{X_i^N}}{P_1}> \gamma + \varepsilon_{i,N} \big\} \big) \nonumber \\
 &\leq \sum_{k\geq N^2 }  \exp\left(|\mathcal{X}|\log\left(k+1\right)\right) \exp\left(-\frac{k}{\sqrt{N}}\right) \label{eq:err_prob_line1}\\
& \leq \frac{M\left(M-1\right)}{2} \exp\left(-\frac{N^2}{\sqrt{N}}\left(1+o\left(1\right)\right)\right)\\
& = \frac{M\left(M-1\right)}{2} \exp\left(-N^{\frac{3}{2}}\left(1+o\left(1\right)\right)\right) \label{eq:err_prob_linef}
\end{align}
where in (\ref{eq:err_prob_line1}) we have used (\ref{eq:stop-upper-bound-start})-(\ref{eq:stop_m_upper_bound_last}) and the fact that $\theta_{N,1\left(i^{\star}\right)}^{-}\geq {N}/{N^2}$.  We obtain
\begin{align}
&\probm_1\left(T^{(M)}_{\text{seq}}\geq N^2\right)\leq\nonumber\\
& \mathbb{E}_1\big[\probm_1\left( T^{(M)}_{\text{seq}}\geq N^2 \Big|   \{\ED{X_j^N}\}_{2\leq j\leq M} \right)\times \nonumber \\
& \hspace{2cm} \mathbbm{1}\lbrace\bigcap_{i=2}^{M} \left\lbrace \KL{\ED{X_i^N}}{P_1}> \gamma + \varepsilon_{i,N} \right\rbrace \rbrace\big]  \nonumber \\
&+ \probm_1\left( \bigcup_{i=2}^{M} \left\lbrace \KL{\ED{X_i^N}}{P_1}\leq\gamma + \varepsilon_{i,N} \right\rbrace\right) \nonumber \\
& \leq  \frac{M\left(M-1\right)}{2} \exp\left(-N^{\frac{3}{2}}\left(1+o\left(1\right)\right)\right) + \nonumber \\
&\sum_{i=2}^{M} \exp\left(- N\min_{V \in \mathcal{P}\left(\mathcal{X}\right):  \KL{V}{P_1} \leq \gamma + \varepsilon_{i,N}} \KL{V}{P_i}\right) \label{eq:error_multi_2}
\end{align}
Plugging (\ref{eq:error_multi_2}) into (\ref{eq:error_multi_22}), we get
\begin{align}
&\pe_1(\Phi_{\text{seq}}^{(M)}(\gamma)) 
 \leq  N^2 \exp\left(-\gamma N\right) \left(N+N^2+1\right)^{|\mathcal{X}|} \nonumber \\
& +   \frac{M\left(M-1\right)}{2} \exp\left(-N^{\frac{3}{2}}\left(1+o\left(1\right)\right)\right) \nonumber\\
 &\qquad  +  \sum_{i=2}^{M} \exp\left(-N \min_{V \in \mathcal{P}\left(\mathcal{X}\right):  \KL{V}{P_1} \leq \gamma + \varepsilon_{i,N}}  \KL{V}{P_i}\right)\\
&\dotleq  \exp\big(-N\min\big\{ \gamma, \min_{i \in [M]\setminus\{1\}}  \min_{V \in \mathcal{P}\left(\mathcal{X}\right):  \KL{V}{P_1} \leq \gamma + \varepsilon_{i,N}} \KL{V}{P_i} \big\}\big)
\end{align}
\end{proof}
\end{lemma}

Using  Lemmas \ref{lem:stop_time_multi} and \ref{lem:error_multi}, we can characterize the achievable error exponent of $\Phi_{\text{seq}}^{(M)}(\gamma)$ as follows
\begin{align}
&\textsf{e}_i\left(\Phi_{\text{seq}}^{(M)}(\gamma)\right)=\liminf\limits_{N\to \infty}\frac{-\log \pe_i(\Phi_{\text{seq}}^{(M)}(\gamma)) }{\mathbb{E}_i\left[T_{\text{seq}}^{(M)}\right]} \nonumber \\
&\geq  \min_{j\in \left[M\right],j\neq i}\{\theta^{\star}_{i(j),\gamma}\} \times \nonumber\\ 
& \hspace{0.3cm} \liminf\limits_{N\to \infty}\min\left\lbrace \gamma, \min_{j \in [M]\setminus\{i\}}  \min_{V :  \KL{V}{P_i} \leq \gamma + \varepsilon_{j,N}} \KL{V}{P_j} \right\rbrace \label{eq:error_exp_multi_1}  \\
&= \min_{j\in[M]\setminus\{  i\}}\{\theta^{\star}_{i(j),\gamma}\} \nonumber\\ 
&\hspace{1.5cm} \times \min\left\lbrace \gamma, \min_{j \in  [M]\setminus\{  i\}}  \min_{V :  \KL{V}{P_i} \leq \gamma} \KL{V}{P_j} \right\rbrace \label{eq:error_exp_multi_2}\\
&=\min_{j\in[M]\setminus\{  i\}}\{\theta^{\star}_{i(j),\gamma}\} \gamma  \label{eq:error_exp_3}\\
&= \min_{j\in[M]\setminus\{  i\}} \mathrm{GJS}\left(P_i,P_j,\theta^{\star}_{i(j),\gamma}\right) . \label{eqn:last_step}
\end{align}
where in (\ref{eq:error_exp_multi_1}) we use Lemma \ref{lem:error_multi}. Then,  (\ref{eq:error_exp_multi_2}) is obtained using the fact that the optimal value is a continuous function of $\varepsilon_{j,N}$, and $\varepsilon_{j,N}$ converges to zero as $N \to \infty$. We have (\ref{eq:error_exp_3}) because
\begin{equation}
\begin{aligned}
\label{eq:chernoof_multi}
&\big\{\gamma  \,\big| \,\bigcap\limits_{i=1}^{M} \big\{ \min_{j \in \left[M\right],j\neq i}  \min_{V \in \mathcal{P}\left(\mathcal{X}\right):  \KL{V}{P_i} \leq \gamma} \KL{V}{P_j} \geq \gamma \big\}\big\} \\
&= [0, \min\limits_{(i,j)\in \mathcal{M}} C\left(P_i,P_j\right)].
\end{aligned}
\end{equation}
where $\mathcal{M}\triangleq \{\left(i,j\right) \in \left[M\right]^2,i\neq j  \}$. Finally the last step in~\eqref{eqn:last_step} follows from (\ref{eq:fix_point_multi}). Thus, we conclude that the achievable error exponent is obtained as stated in Corollary~\ref{corr:achiev-multi}.

\begin{appendices}
\section{Proof of Lemma \ref{lem:alpha_p_n}}
\label{app:theta_+_-}
\begin{lemma} \label{lem:alpha+}
Let $\lambda>0$  and let $X^N$ be a sequence drawn from the product distribution $Q^{N}$. Also, let $\alpha_N^{\star}$ satisfy the equation $\GJS{\ED{X^N}}{P}{\alpha_N^{\star}}=\lambda \alpha_N^{\star}$. Consider the optimization problem 
\begin{equation}\label{eq:opt_lower_bound}
\begin{aligned}
U\left(\alpha\right) ~ \triangleq ~
 & \underset{V \in \mathcal{P}\left(\mathcal{X}\right)}{\text{max}}
& & \GJS{\ED{X^N}}{V}{\alpha} \\
& \text{s.t.}
& &  \KL{V}{P} \leq \frac{1}{\sqrt{N}}.
\end{aligned}
\end{equation}
Then, $
\alpha^{+}_N  = \alpha_{N}^{\star} + O\left(\frac{\log N}{N^{\frac{1}{4}}}\right)$
%
satisfies the following inequality
$ U\left(\alpha^{+}_N\right) < \lambda \alpha^{+}_N$.
\begin{proof}
We begin the proof by rewriting the objective function as  
\begin{align}
&\GJS{\ED{X^N}}{P}{\alpha}  \nonumber\\
&=\min_{W \in \mathcal{P}\left(\mathcal{X}\right) } \big(\sum_{z\in \mathcal{X}}\left(\alpha \ED{X^N}\left(z\right)+V(z)\right)\log 1/W\left(z\right)\big)\nonumber \\
& -\alpha \ent{\ED{X^N}} - \ent{V}\\
&\leq -\left(\sum_{z \in \mathcal{X}}\left(\alpha \ED{X^N}\left(z\right)+V\left(z\right)\right)\log \frac{\alpha \ED{X^N}\left(z\right) + P\left(z\right) }{1+\alpha}\right) \nonumber\\
&+ \alpha \sum_{z\in \mathcal{X}} \ED{X^N}\left(z\right)\log \ED{X^N}\left(z\right) + \sum_{z \in \mathcal{X}} V\left(z\right)\log V\left(z\right)\label{eq:set_p2}
\end{align}
where the first step is obtained by using Lemma \ref{lem:def_gjs} where we show that GJS can be written in the form of an optimization problem. Setting $W=P$ in the second step, we find an upper bound on the objective function of (\ref{eq:opt_lower_bound}). Then, we obtain
\begin{align}
&\GJS{\ED{X^N}}{P}{\alpha}\leq \GJS{\ED{X^N}}{P}{\alpha}  \nonumber\\
& + \KL{V}{\frac{\alpha \ED{X^N}+P}{1+\alpha}} - \KL{P}{\frac{\alpha \ED{X^N}+P}{1+\alpha}} \label{eq:manip1}\\
&\leq  \GJS{\ED{X^N}}{P}{\alpha_N^{\star}} + \KL{\ED{X^N}}{\frac{\alpha_N^{\star}\ED{X^N}+P}{1+\alpha_N^{\star}}} \left(\alpha-\alpha_N^{\star}\right) \nonumber\\
& + \sum_{z \in \mathcal{X}} \left(P\left(z\right)-V\left(z\right)\right)\log \frac{\alpha \ED{X^N}\left(z\right)+P\left(z\right)}{1+\alpha}+ \ent{P}- \ent{V} \label{eq:last_step}\\
&\leq  \GJS{\ED{X^N}}{P}{\alpha_N^{\star}} + \KL{\ED{X^N}}{\frac{\alpha_N^{\star}\ED{X^N}+P}{1+\alpha_N^{\star}}} \left(\alpha-\alpha_N^{\star}\right) \nonumber\\
& + \sum_{z \in \mathcal{X}} \left(P\left(z\right)-V\left(z\right)\right)\log \frac{\alpha \ED{X^N}\left(z\right)+P\left(z\right)}{1+\alpha} \nonumber \\
&- \|P - V\|_1 \log\frac{\|P_2 - V\|_1}{|\mathcal{X}|}.
\end{align}
 Equation (\ref{eq:manip1}) follows from the definitions of GJS and KL divergences. Step (\ref{eq:last_step}) comes from the fact that GJS is a concave function in $\alpha$ as shown in Lemma \ref{lem:g}. Finally, in the last step we have used \cite[Thm. 17.3.3]{cover}.
Therefore, we have
\begin{align} \label{eq:before_final_step_}
&\max_{V\in \mathcal{P}\left(\mathcal{X}\right): \KL{V}{P}\leq \frac{1}{\sqrt{N}}} \GJS{\ED{X^N}}{V}{\alpha} \nonumber\\
&\leq  \max_{V\in \mathcal{P}\left(\mathcal{X}\right): \KL{V}{P}\leq \frac{1}{\sqrt{N}}}\sum_{z \in \mathcal{X}} \left(P\left(z\right)-V\left(z\right)\right)\log \frac{\alpha \ED{X^N}\left(z\right)+P\left(z\right)}{1+\alpha} \nonumber\\
&+\frac{\sqrt{2}}{N^{\frac{1}{4}}}\log\left(\frac{|\mathcal{X}|N^{\frac{1}{4}}}{\sqrt{2}}\right)+\GJS{\ED{X^N}}{P}{\alpha_N^{\star}} \nonumber \\
&+ \KL{\ED{X^N}}{\frac{\alpha_N^{\star}\ED{X^N}+P}{1+\alpha_N^{\star}}} \left(\alpha-\alpha^{\star}_N\right), 
\end{align}
where (\ref{eq:before_final_step_}) is because Pinsker's inequality \cite[Lemma 11.6.1]{cover} and $\KL{V}{P}\leq \frac{1}{\sqrt{N}}$. Considering the optimization problem in the RHS of (\ref{eq:before_final_step_}), we need to provide an upperbound for
\begin{equation}
\label{eq:opt_lower_bound2}
\begin{aligned}
& \underset{V \in \mathcal{P}\left(\mathcal{X}\right)}{\text{max}}
& & \sum_{z\in \mathcal{X}} \left(P\left(z\right)-V\left(z\right)\right)\log \frac{\alpha\ED{X^N}\left(z\right)+P\left(z\right)}{1+\alpha}\\
& \text{s.t.}
& & \|V-P \|_1\leq \frac{\sqrt{2}}{N^{\frac{1}{4}}},
\end{aligned}
\end{equation}
Let us define $\epsilon_z \triangleq V\left(z\right) - P\left(z\right)$ for all $z \in \mathcal{X}$. We can rewrite the optimization problem in (\ref{eq:opt_lower_bound2}) as 
\begin{subequations}
\label{opt:main0}
\begin{align}
    \underset{\bm{\epsilon}:\sum_{z \in \mathcal{X}} \epsilon_z = 0}{\text{max}}
        & \quad -\sum_{z\in \mathcal{X}} \epsilon_z \log \frac{\alpha\ED{X^N}\left(z\right)+P\left(z\right)}{1+\alpha}\\
    \text{s.t.} 
        & \quad \sum_{z \in \mathcal{X}} |\epsilon_z|\leq \frac{\sqrt{2}}{N^{\frac{1}{4}}} \label{eq:const_kh}\\
         & \quad -P\left(z\right) \leq \epsilon_z \leq 1-P\left(z\right) \quad \forall z \in \mathcal{X} \label{eq:const_alaki}
\end{align}
\end{subequations}
Because $\min_{z\in \mathcal{Z}}P\left(z\right)> 0$, as $N$ becomes large, it is straightforward to verify that the constraints in (\ref{eq:const_alaki}) will not hold with equality at the optimal point since if so, this would contradict (\ref{eq:const_kh}). Thus, we can omit the constraint in (\ref{eq:const_alaki}). With this simplification, the optimization problem in (\ref{opt:main0}) is in the form of that in Lemma~\ref{lem:linear_opt}, and the optimal value is given by 
\begin{equation}
\frac{\sqrt{2}}{N^{\frac{1}{4}}}\log\frac{\max_{z \in \mathcal{X}} \{\alpha \ED{X^N}\left(z\right)+P\left(z\right)\}}{\min_{z \in \mathcal{X}} \{\alpha \ED{X^N}\left(z\right)+P\left(z\right)\}}.
\end{equation}
We can further upper bound the optimal value as
\begin{equation}\label{eq:opt_val_upper_1}
\begin{aligned}
&\frac{\sqrt{2}}{N^{\frac{1}{4}}}\log\frac{\max_{z \in \mathcal{X}} \{\alpha \ED{X^N}\left(z\right)+P\left(z\right)\}}{\min_{z \in \mathcal{X}} \{\alpha \ED{X^N}\left(z\right)+P\left(z\right)\}}\leq \nonumber\\
&\frac{\sqrt{2}}{N^{\frac{1}{4}}}\log \left(\frac{\max_{z\in \mathcal{X}}P\left(z\right)}{\min_{z\in \mathcal{X}}P\left(z\right)} \right)+ \frac{\sqrt{2}}{N^{\frac{1}{4}}} \frac{\max_{z\in \mathcal{X}}\ED{X^N}\left(z\right)}{\max_{z\in \mathcal{Z}}P\left(z\right)} \alpha
\end{aligned}
\end{equation}
Therefore, plugging (\ref{eq:opt_val_upper_1}) into (\ref{eq:before_final_step_}) we can provide an upper bound for the optimal value of (\ref{eq:opt_lower_bound}). Finally, letting the upper bound be less than $\lambda \alpha$,  we obtain the desired result.
\end{proof}
\end{lemma}
\begin{lemma} \label{lem:opt_upper_bound}
Let
\begin{equation}\label{eq:opt_upper_bound}
\begin{aligned}
U_{-}\left(\alpha\right)~ \triangleq ~ & \underset{V\in \mathcal{P}\left(\mathcal{X}\right)}{\text{min}}
& & \GJS{\ED{X^N}}{V}{\alpha} \\
& \text{s.t.}
& & \KL{V}{P}\leq \frac{1}{\sqrt{N}}.
\end{aligned}
\end{equation}
Then, we have 
$U_{-}\left(\alpha_N^-\right) \geq \lambda \alpha_N^-,$ where
\begin{equation}
\label{eq:alpha_-_khodesh}
\alpha_N^-=\alpha_{N}^{\star} - O\left(\frac{1}{N^{\frac{1}{4}}}\right)
\end{equation}

\end{lemma}
\begin{proof}
In Lemma \ref{lem:g}, we proved that $\mathrm{GJS}$ is a convex function in its second argument. Therefore, we can write 
\begin{align}
&\GJS{\ED{X^N}}{V}{\alpha}\geq \GJS{\ED{X^N}}{P}{\alpha}\nonumber \\
& + \sum_{z\in \mathcal{X}} \log\frac{(1+\alpha)P\left(z\right)}{\alpha \ED{X^N}\left(z\right) + P\left(z\right) }\left(V\left(z\right)-P\left(z\right)\right)\label{eq:lowe_bound1},
\end{align}
where we have used the fact that for a convex function $f$, we have $f(x)\ge f(y)+\nabla f (y)^{T} ( x-y)$ for all $x$ and $y$. Plugging~(\ref{eq:lowe_bound1}) into (\ref{eq:opt_upper_bound}), we arrive at the following optimization problem:
\begin{equation}\label{eq:opt_linear_2}
\begin{aligned}
& \underset{V\in \mathcal{P}\left(\mathcal{X}\right)}{\text{min}}
& &\sum_{z \in \mathcal{X}} \log\frac{(1+\alpha)P\left(z\right)}{\alpha \ED{X^N}\left(z\right) + P\left(z\right) }\left(V\left(z\right)-P\left(z\right)\right) \\
&  \text{s.t.}
& & \|V-P\|_1\leq \frac{\sqrt{2}}{N^{\frac{1}{4}}}.
\end{aligned}
\end{equation}
Here, in (\ref{eq:opt_linear_2}) we have     Pinsker's inequality \cite[Lemma 11.6.1]{cover} in the first constraint. Using Lemma \ref{lem:linear_opt}, it directly follows that the optimal value of the optimization problem (\ref{eq:opt_linear_2}) is 
\begin{equation}
\frac{1}{\sqrt{2}N^{\frac{1}{4}}} \log \left(\frac{\min_{z\in \mathcal{X}} \lbrace \frac{P\left(z\right)}{\alpha \ED{X^N}\left(z\right) + P\left(z\right) } \rbrace}{\max_{z\in \mathcal{X}} \lbrace \frac{P\left(z\right)}{\alpha \ED{X^N}\left(z\right) + P\left(z\right) } \rbrace}\right)
\end{equation}
which is straightforward to show that the optimal value can be lower bounded by
\begin{equation} \label{eq:lower_bound_22}
\frac{1}{\sqrt{2}N^{\frac{1}{4}}} \log \frac{\min_{z\in \mathcal{X}}P\left(z\right)}{\max_{z\in \mathcal{X}}P\left(z\right)}- \frac{\alpha}{\sqrt{2}N^{\frac{1}{4}}} \frac{\max_{z\in \mathcal{X}}\ED{X^N}\left(z\right)}{\max_{z\in \mathcal{X}}P\left(z\right)}.
\end{equation}
Plugging (\ref{eq:lower_bound_22}) into (\ref{eq:lowe_bound1}), we obtain
\begin{align}
&\min_{V\in \mathcal{P}\left(\mathcal{X}\right):\KL{V}{P}\leq \frac{1}{\sqrt{N}}} \GJS{\ED{X^N}}{V}{\alpha}\\
&\geq \GJS{\ED{X^N}}{P}{\alpha} + \frac{1}{\sqrt{2}N^{\frac{1}{4}}} \log \frac{\min_{z\in \mathcal{X}}P\left(z\right)}{\max_{z\in \mathcal{X}}P\left(z\right)} \nonumber \\
& - \frac{\alpha}{\sqrt{2}N^{\frac{1}{4}}} \frac{\max_{z\in \mathcal{X}}\ED{X^N}\left(z\right)}{\max_{z\in \mathcal{X}}P\left(z\right)}. \label{eq:first_low}
\end{align}
Fix $0 \leq \theta \leq \alpha_N^{\star}$. From Taylor's theorem, there exists an $\tilde{\theta}\in \left(\alpha_N^{\star}-\theta,\alpha_N^{\star}\right)$ such that
\begin{align}
&\GJS{\ED{X^N}}{P}{\alpha_N^{\star}-\theta} \nonumber\\
&= \GJS{\ED{X^N}}{P}{\alpha_N^{\star}} - \KL{\ED{X^N}}{\frac{\alpha_N^{\star}\ED{X^N}+P}{1+\alpha_N^{\star}}} \theta \nonumber\\
& + \frac{\theta^2}{2(1+\tilde{\theta})}\sum_{z\in \mathcal{X}} \ED{X^N}\left(z\right) \frac{P\left(z\right) -  \ED{X^N}\left(z\right)}{\tilde{\theta} \ED{X^N}\left(z\right) + P\left(z\right) }\\
&\geq \GJS{\ED{X^N}}{P}{\alpha_N^{\star}} - \KL{\ED{X^N}}{\frac{\alpha_N^{\star}\ED{X^N}+P}{1+\alpha_N^{\star}}} \theta \nonumber\\
& + \frac{\theta^2}{2} \left(\frac{1}{1+\alpha_N^{\star}}-\sum_{z\in \mathcal{X}}\frac{\ED{X^N}\left(z\right)^2}{P\left(z\right)}\right).
\end{align}
Here, the final step follows by lower bounding the second derivative term. Finally letting the lower bound in (\ref{eq:first_low}) be smaller $\lambda \left(\alpha-\theta\right)$, we need to find $\theta$ such that 
\begin{equation}
\label{eq:quadratic_eq}
\begin{aligned}
&\frac{\theta^2}{2}\left(\frac{1}{1+\alpha_N^{\star}}-\sum_{z\in \mathcal{X}}\frac{\ED{X^N}\left(z\right)^2}{P\left(z\right)}\right)+\theta\lambda\\
&\theta\big(- \KL{\ED{X^N}}{\frac{\alpha_N^{\star}\ED{X^N}+P}{1+\alpha_N^{\star}}}+\frac{1}{\sqrt{2}N^{\frac{1}{4}}} \frac{\max_{z\in \mathcal{X}}\ED{X^N}\left(z\right)}{\max_{z\in \mathcal{X}}P\left(z\right)}\big)\\
&+  \frac{1}{\sqrt{2}N^{\frac{1}{4}}} \log \frac{\min_{z\in \mathcal{X}}P\left(z\right)}{\max_{z\in \mathcal{X}}P\left(z\right)} -\frac{\alpha_N^{\star}}{\sqrt{2}N^{\frac{1}{4}}} \frac{\max_{z\in \mathcal{X}}\ED{X^N}\left(z\right)}{\max_{z\in \mathcal{X}}P\left(z\right)} =0
\end{aligned}
\end{equation}
Finally, considering (\ref{eq:quadratic_eq}) is a quadratic equation in $\theta$ and $\alpha_N^{-}=\alpha_N^{\star}-\theta$, we obtain the desired result.
\end{proof}

\end{appendices}
\bibliographystyle{unsrt}
\bibliography{refbc}

\begin{IEEEbiographynophoto}{Mahdi Haghifam}
was born in Iran in 1992. He received the B.Sc. and M.Sc. degrees in electrical engineering in 2014 and 2016, respectively, from Sharif University of Technology, Tehran, Iran. Since September 2017, he has been pursuing the Ph.D. degree with the electrical and engineering department at University of Toronto, Toronto, Canada. His research interests include different aspects of Machine Learning and Information Theory, specially applications of the latter in the former.
\end{IEEEbiographynophoto}

\begin{IEEEbiographynophoto}
{Vincent Y.\ F.\ Tan} (S'07-M'11-SM'15)   was born in Singapore in 1981. He is currently a Dean's Chair Associate Professor in the Department of Electrical and Computer Engineering  and the Department of Mathematics at the National University of Singapore (NUS). He received the B.A.\ and M.Eng.\ degrees in Electrical and Information Sciences from Cambridge University in 2005 and the Ph.D.\ degree in Electrical Engineering and Computer Science (EECS) from the Massachusetts Institute of Technology (MIT)  in 2011.  His research interests include information theory, machine learning, and statistical signal processing.

Dr.\ Tan was also an IEEE Information Theory Society Distinguished Lecturer for 2018/9. He is currently serving as an Associate Editor of the {\em IEEE Transactions on Signal Processing} and an Associate Editor of Machine Learning for the {\em IEEE Transactions on Information Theory}. He is a member of the IEEE Information Theory Society Board of Governors. 
\end{IEEEbiographynophoto}

\begin{IEEEbiographynophoto}{Ashish Khisti}
received the B.ASc. degree from the Engineering Science Program, University of Toronto, in 2002, and the masterâ€™s and
Ph.D. degrees from the Department of Electrical Engineering and Computer
Science, Massachusetts Institute of Technology (MIT), Cambridge, MA, USA,
in 2004 and 2008, respectively. Since 2009, he has been on the Faculty
in the Electrical and Computer Engineering (ECE) Department, University
of Toronto, where he was an Assistant Professor from 2009 to 2015,
an Associate Professor from 2015 to 2019, and is currently a Full Professor.
He also holds a Canada Research Chair in information theory with the ECE
Department. His current research interests include theory and applications
of machine learning and communication networks. He is also interested in
interdisciplinary research involving engineering and healthcare
\end{IEEEbiographynophoto}

\end{document}